\documentclass{article} 
\usepackage{iclr2027_conference,times}

\usepackage{amsmath,amsfonts,bm}

\def\eqref#1{equation~\ref{#1}}

\def\1{\bm{1}}

\DeclareMathAlphabet{\mathsfit}{\encodingdefault}{\sfdefault}{m}{sl}
\SetMathAlphabet{\mathsfit}{bold}{\encodingdefault}{\sfdefault}{bx}{n}

\newcommand{\E}{\mathbb{E}}

\newcommand{\KL}{D_{\mathrm{KL}}}

\usepackage{hyperref}
\usepackage{url}

\usepackage{amsmath,amsthm,mathtools}
\usepackage{newtxtext,newtxmath}
\usepackage{booktabs,tabularx,longtable}
\usepackage{algorithm,algpseudocode}
\usepackage{enumitem,microtype,hyperref,url,etoolbox}
\hypersetup{hidelinks}
\newtheorem{theorem}{Theorem}
\newtheorem{lemma}[theorem]{Lemma}
\newtheorem{proposition}[theorem]{Proposition}
\newtheorem{corollary}[theorem]{Corollary}
\theoremstyle{definition}
\newtheorem{definition}[theorem]{Definition}

\algrenewcommand\algorithmicrequire{\textbf{Require}}
\algrenewcommand\alglinenumber[1]{\footnotesize #1.}
\newcommand{\PaperURL}[1]{\href{#1}{Online source}}
\newcommand{\TV}{\operatorname{TV}}
\newcommand{\cost}{\operatorname{Cost}}
\newcommand{\lref}{\ell^{\mathrm{ref}}}
\newcommand{\hresp}{h_{\mathrm{resp}}}
\newcommand{\Jna}{J_A^{\mathrm{nat}}}
\newcommand{\Jra}{J_A^{\mathrm{ref}}}
\newcommand{\Jsel}{J^{\mathrm{sel}}}
\newcommand{\Jav}{J_\mu^{\mathrm{av}}}
\newcommand{\Jal}{J_\mu^{\mathrm{al}}}
\newcommand{\Jcl}{J_\mu^{\mathrm{cl}}}
\newcommand{\Jpa}{J_\mu^{\mathrm{path}}}
\newcommand{\Jrm}{J_\mu^{\mathrm{ref}}}
\newcommand{\Jnm}{J_\mu^{\mathrm{nat}}}

\title{Counterfactual Regret for Reflective Agents with Lean-Verified Guarantees}

\author{Zhongnian Li, Shi Xu, Mengting Xu, Peng Ying, Xinzheng Xu\thanks{Corresponding author: \texttt{xxzheng@cumt.edu.cn}}\\
School of Computer Science and Technology /
School of Artificial Intelligence\\
China University of Mining and Technology\\
Xuzhou, China\\
\texttt{\{zhongnianli,xushi,mengtingxu,pengying,xxzheng\}@cumt.edu.cn}}

\iclrfinalcopy

\begin{document}

\maketitle
\pagestyle{plain}

\begin{abstract}
Reflective language agents reuse feedback and memory to adapt across tasks without updating model weights. However, existing guarantees for policy selection and action generation do not by themselves establish when local reflection yields long-term performance guarantees for the complete agent under resource constraints. In this paper, we address this gap by deriving sufficient conditions for sublinear counterfactual regret from local generation and deployment guarantees, while retaining the native controller. Specifically, we establish a resource-sensitive transfer that bounds full-state deployment discrepancies along actual execution paths and preserves the own-history laws of both learner and comparator. Besides, local generation reliability yields a cumulative generation opportunity bound, linking policy acquisition to long-term performance through an explicit quality, coverage, and hard-budget tradeoff. To verify these guarantees, we formalize the conditional proof chain for all thirteen numbered theoretical results in Lean, including a native instance with repeated policy growth and imperfect validation. 
The complete Lean proof artifacts are available in an anonymous repository\footnote{\url{https://github.com/subeidiyixiaoqianfeng/counterfactual-regret-lean-proofs}} for verification and reproduction.
\end{abstract}

\section{Introduction}
Reflective language agents use feedback to revise decisions and retain experience for subsequent tasks. With frozen model weights, linguistic reflection and memory provide a practical means of adaptation during execution. Reflexion and Self-Refine demonstrate benefits in coding, reasoning, and iterative task solving~\citep{shinn2023reflexion,madaan2023selfrefine}. Reusable instructions can carry lessons from earlier outcomes into later decisions, making reflection valuable for systems that must adapt through interaction. These capabilities motivate understanding how local feedback-driven changes contribute to long-term agent performance.

However, existing guarantees for policy selection and action generation do not by themselves establish sustained performance gains for the complete reflective process. With frozen model weights, useful reflection must produce reusable policies and influence subsequent execution within available resources. These effects are coupled because current decisions shape future feedback and policy availability. Understanding this coupling requires explicit conditions linking local adaptation to long-term performance. We study this gap through counterfactual regret against resource-feasible procedures on their own histories, following the policy-regret comparison principle~\citep{arora2012}. Rather than requiring every reflection to improve performance, we identify generation and deployment conditions under which average excess loss vanishes while retaining the native controller.

\paragraph{Contributions.}
We develop a conditional theory connecting local reflection mechanisms to a complete native-agent regret guarantee. Theorem~\ref{thm:local} retains the native controller and identifies conditions for sublinear regret through three contributions.
\begin{enumerate}[leftmargin=*,nosep]
\item \emph{Resource-sensitive native transfer.} Full-state deployment discrepancies are bounded along actual resource paths, preserving both counterfactual marginals and connecting hard budgets to cumulative loss.
\item \emph{Locally derived generation bounds.} Local reliability connects policy acquisition to long-term performance, with an explicit quality, coverage and hard-budget tradeoff.
\item \emph{A Lean-verified conditional proof chain.} All thirteen numbered results are formalized, including a concrete native instance that derives compatible interfaces and a sublinear bound from its mechanisms.
\end{enumerate}

\paragraph{Related work.}
Growing-expert tracking provides guarantees for expanding pools~\citep{mourtada2017}, while generative action sets jointly optimize creation and selection through a context-conditioned oracle and quadratic mismatch model~\citep{xu2026gas}. We connect local generation reliability to complete reflective executions through a generation opportunity bound. Its composition with deployment stability 
captures how obtaining useful policies and executing them jointly affect regret.

Policy regret formalizes action-responsive comparisons~\citep{arora2012}, with recent Markov-game guarantees under function-approximation conditions~\citep{nguyentang2025}. We develop a resource-sensitive transfer for a retained native controller. Full-state perturbations accumulate along actual execution paths, with hard budgets controlling their cumulative effect. The construction preserves both native marginals while connecting local deployment conditions to own-history regret.

\citet{merten2018} established Coq-verified multiplicative weights and composed distributed-system guarantees. Our Lean development connects growing selection, local generation reliability, deployment stability, and completed-task correspondence through explicit interfaces, supporting reproducible checking of the composed no-regret guarantee. 

\section{Model and reference algorithm}
\label{sec:model}
We define the complete reflective execution before introducing a reference process whose selection guarantee can be transferred to it.
\subsection{Histories, policies, and counterfactual regret}
Use task indices $t=0,\ldots,T-1$, internal-operation indices $s\ge0$, and reference-slot indices $r=0,\ldots,R$. These clocks differ. All logarithms are natural, $\mathbf1\{E\}$ is the indicator of event $E$, and $\langle u,v\rangle=\sum_i u(i)v(i)$. Appendix~\ref{app:notation} collects notation and its Lean correspondence.

\begin{definition}[Legal history and operations]
\label{def:history}
Fix measurable operation and result spaces $\mathsf U,\mathsf O$, an initial-history set $\mathcal H^{\rm op}_0$, and a measurable admissibility relation $\mathsf E(h,u,o)$. Define $\mathcal H^{\rm op}_{s+1}=\{h\mathbin{\|}(u,o)\mid h\in\mathcal H^{\rm op}_s,\ \mathsf E(h,u,o)\}$ and $\mathcal H^{\rm op}=\bigsqcup_{s\ge0}\mathcal H^{\rm op}_s$, assuming each level measurable. Here $\|$ appends a completed operation and result. A measurable legal-operation set $\mathsf U(h)$ restricts the next control. Histories record contexts, primitive actions, structured outcomes, memory, versioned registrations, and resource counters. Execution can name only an already registered executable identity. Legal prefixes need not have positive probability. Hard budgets impose~\eqref{eq:hard}.
\end{definition}

A history records what has happened. An execution policy specifies the distribution of the next primitive action, given the current context and memory.

\begin{definition}[Execution policy and card identity]
\label{def:card}
Let $\mathsf X,\mathsf M,\mathsf A$ be measurable context, memory, and primitive-action spaces. Identity $i$ names an immutable card program and interpreter version, inducing a probability kernel $\pi_i(da\mid x,m)$ on $\mathsf A$. For every measurable $D\subseteq\mathsf A$, $\pi_i(D\mid x,m)$ is measurable in $(x,m)$, and $\pi_i(\mathsf A\mid x,m)=1$. Selecting $i$ chooses a kernel, whereas sampling $a$ executes an action. Memory changes may alter its output law without changing identity. Program or interpreter changes require a new identity.
\end{definition}

A meta-policy governs the complete procedure, including when to create cards and which existing card to execute. It is therefore more than a selector over a fixed library.

\begin{definition}[Meta-policy and own-history law]
\label{def:metapolicy}
A meta-policy $\mu(du\mid h)$ is a measurable probability kernel on operations with $\mu(\mathsf U(h)\mid h)=1$. Generation, validation, memory updates, and execution are distinct operations. The history index allows time-dependent control. Composing this kernel with environment and runtime kernels, assuming internal operations terminate, induces task records $Z_t\in\mathcal Z_t$. Here $Z_0$ is the prescribed initial record and $Z_{t+1}$ completes task $t$. The task history is $H_t=(Z_0,\ldots,Z_t)\in\mathcal H_t^{\rm task}=\prod_{u=0}^t\mathcal Z_u$. Its recursively generated law is $\mathbb P_\mu$. A nonempty class $\mathcal M$ contains complete procedures, not distributions over current cards. The native learner $A$ is another procedure and need not belong to $\mathcal M$.
\end{definition}

The \emph{native controller} is the operation-selection kernel $\mu_A(du\mid h)$ of procedure $A$, with $\mu_A(\mathsf U(h)\mid h)=1$. Operation-result, environment, and runtime kernels additionally determine $\mathbb P_A$. Retaining this controller means that it need not execute the reference algorithm. Decisions may change across histories, and transfer conditions do not hold automatically.

Fix a common measurable loss scorer $\sigma_t$ on $\mathcal H_{t+1}^{\rm task}$ and set $Y_t=\sigma_t(H_{t+1})\in[0,1]$, with smaller values preferred. This scalar evaluates the structured output, rather than replacing it. Define
\begin{equation}
 L_T^A=\sum_{t<T}Y_t^A,\qquad J_T(A)=\E_{\mathbb P_A}L_T^A,\qquad
 R_T(A,\mathcal M)=J_T(A)-\inf_{\mu\in\mathcal M}J_T(\mu).
 \label{eq:objective}
\end{equation}
The expectation includes environment, generation, validation, runtime, logging, and private control randomness present in the specified execution. A deterministic environment is permitted. The native model may respond causally to past actions. Oblivious loss tables are required only in the benchmark of Section~\ref{sec:benchmark}.

For a measurable trajectory map $r_A$ and source law $\Gamma$, $((r_A)_\#\Gamma)(D)=\Gamma(r_A^{-1}(D))$ for every measurable event $D$. The subscript $\#$ denotes \emph{measure pushforward}, so $(r_A)_\#\Gamma=\mathbb P_A$ fixes the native trajectory law.

\begin{proposition}[Well-defined native objective]
\label{prop:wellposed}
The loss in~\eqref{eq:objective} is measurable and integrable, with $0\le J_T(A)\le T$. A measurable realization $r_A$ on a coupling space with $(r_A)_\#\Gamma=\mathbb P_A$ satisfies $\E_\Gamma[L_T^A\circ r_A]=J_T(A)$.
\end{proposition}
Proposition~\ref{prop:wellposed} preserves the full native trajectory law, not merely the current action distribution. This is required for every comparator.

\paragraph{The counterfactual is a complete execution.}
Both procedures use the same environment mechanism and scoring function, but supply their own preceding actions and memories to that mechanism. Equal initial conditions do not mean equal subsequent observations. Randomness may be coupled only with the correct marginals. In particular, the comparator cannot use a card solely because the learner generated it. The class restriction is fixed before taking the infimum. This is the distinction between comparison on a realized loss sequence and policy-regret comparison in a responsive system~\citep{arora2012,anava2015}.

\subsection{Resource constraints and interfaces}
Require a nonempty comparator class satisfying, under each own law,
\begin{equation}
 \mathbb P_\mu\{\cost_t^\mu\le B_t\text{ for every }t=0,\ldots,T\}=1.
 \label{eq:hard}
\end{equation}
Here $\cost_t^\mu$ is cumulative resource use after $t$ completed tasks, including failed internal attempts, and $B_t$ is a deterministic prefix budget. Any additional operationwise limit belongs to $\mathsf U(h)$. Hard feasibility restricts procedures, not task loss. Mean feasibility instead bounds $\E_{\mathbb P_\mu}\cost_t^\mu$, assuming integrability, and is not equivalent to~\eqref{eq:hard}.

The \emph{resource interface} determines admissible procedures. The \emph{representation interface} maps a comparator to a causal execution-card path and charges any lost advantage. The \emph{deployment interface} compares that represented execution with what the native runtime actually does. These maps need not commute. For example, compilation may consume resources, and imperfect implementation may change future memory. We therefore impose them in a specified order and account for adjacent differences, rather than treating them as independent interchangeable assumptions.

\subsection{Growing-library exponential weights}
\label{sec:selection}
The reference algorithm is a mathematical comparison process, not an installed controller. Write $\mathcal F_r^-$ for completed pre-draw information, with realization $h_r$. The nonempty finite identity set $\mathcal I_r=\mathcal I_r(h_r)$ and sampling law are measurable with respect to this information. Let $w_0=\nu_0$ be a probability distribution and $\Theta_r$ a nonnegative row-stochastic transition from $\mathcal I_r$ to $\mathcal I_{r+1}$. For nonnegative estimated losses $\widehat\ell_r(i)$ and learning rate $\eta>0$, define
\begin{equation}
 p_r(i)=\frac{w_r(i)}{\sum_j w_r(j)},\qquad
 w_{r+1}(j)=\sum_i w_r(i)e^{-\eta \widehat\ell_r(i)}\Theta_r(i,j).
 \label{eq:recursion}
\end{equation}
This recurrence uses the EXP3 and Fixed-Share primitives~\citep{auer2002,herbster1998,cesabianchi2006}. Its normalizers are positive. For persistent identity labels, Fixed Share uses
\begin{equation}
 \Theta_r(i,j)=(1-\alpha)\mathbf1\{i=j\}+\alpha\nu_{r+1}(j),
 \qquad 0<\alpha<1,
 \label{eq:share}
\end{equation}
where $\nu_{r+1}$ has full support and persistent labels share one immutable namespace. With a full-support initial prior, this specialization gives positive draw probabilities, including for newborn cards. A positive normalizer alone does not imply coordinatewise support. Generation is not optimized by this recurrence.

\begin{algorithm}[H]
\caption{Growing-Library Exponential Weights}
\label{alg:reference}
\begin{algorithmic}[1]
\Require Slot cap $N=R+1$, $\eta>0$, prior $\nu_0$, and admissible run protocol
\Require Accepted-score interface supplying $M_r,y_r,q_r$ and positive sampling support
\State Initialize registered identities $\mathcal I_0$ and weights $w_0\gets\nu_0$.
\For{$r=0,\ldots,N-1$}
  \If{slot $r$ is active according to completed history $h_r$}
    \State Set $p_r(i)\gets w_r(i)/\sum_j w_r(j)$ and sample $J_r\sim p_r$.
    \State Execute the admissible scored run and receive the selected outcome $O_r$.
    \State Set $\widehat\ell_r$ by~\eqref{eq:estimate} and settle task and resource costs.
  \Else
    \State Set $\widehat\ell_r\gets0$ and apply the terminal convention of Section~\ref{sec:stopping}.
  \EndIf
  \State Form $\mathcal I_{r+1}$ from completed registrations and choose row-stochastic $\Theta_r$.
  \State Set $w_{r+1}(j)\gets\sum_{i\in\mathcal I_r}w_r(i)e^{-\eta\widehat\ell_r(i)}\Theta_r(i,j)$.
\EndFor
\end{algorithmic}
\end{algorithm}
The protocol supplies lawful run boundaries and acceptance probabilities, not unexecuted losses. A valid arrival closes a scored slot before the next draw. New identities enter only after generation and validation, and can be selected at that next draw. There is no additional draw outside the loop. The terminal weight update is for analysis, not another task. The synchronized run contract in Section~\ref{sec:stopping} remains essential.

\subsection{Feedback model and selection guarantee}
At fixed $h_r$, let $k_r(i,o)$ be the finite outcome law under reference execution of identity $i$, and let $\lambda_r(i,o)$ be its prescribed task-loss total. The target $\lref_r(i)=\sum_o k_r(i,o)\lambda_r(i,o)$ is defined by that execution, not by accepted logs. Let $M_r(i,o)\in\{0,1\}$ mark acceptance and $y_r(i,o)\in[0,d_r]$ be the proxy score, with predictable cap $d_r$. Suppressing $h_r$, require
\begin{equation}
 q_r(i)=\sum_o k_r(i,o)M_r(i,o)>0,\qquad
 \sum_o k_r(i,o)M_r(i,o)y_r(i,o)=q_r(i)\lref_r(i).
 \label{eq:accepted}
\end{equation}
It is stronger than positive logging probability. It excludes selection-biased accepted scores unless an additional correction is proved. The estimator is
\begin{equation}
 \widehat\ell_r(i)=\frac{\mathbf1\{J_r=i\}M_r(J_r,O_r)y_r(J_r,O_r)}{p_r(i)q_r(i)}.
 \label{eq:estimate}
\end{equation}
Only the executed observation enters the update. A structured outcome is compatible with~\eqref{eq:accepted}, because its score is a measurable function. The scalar score does not replace the full state used by reflection.

\paragraph{Covered stopping rules.}
\label{sec:stopping}
Fix a deterministic cap $N=R+1$. The event $r\text{ active}$ means that the reference process has not terminated and completed history $h_r$ authorizes slot $r$. Its indicator $\mathbf1\{r\text{ active}\}$ is $1$ on this event and $0$ otherwise, determined in $\mathcal F_r^-$ before the current draw. This is distinct from card availability $\mathcal I_r$ and feedback acceptance $M_r$. Inactive slots produce an absorbing terminal token and zero targets and estimates. A result may stop later slots, never retroactively cancel its own slot. Held-run length and eligibility require the synchronized accepted-score contract. Arbitrary current-draw-dependent aggregation needs an additional bridge, not merely a shared clock.

\begin{lemma}[Local and stopped observation moments]
\label{lem:stopped}
In a finite observation model satisfying~\eqref{eq:accepted}, with positive $p_r(i)$, the conditional draw--outcome law gives
\begin{equation}
 \E_{J_r,O_r\mid h_r}\widehat\ell_r(i)=\lref_r(i),\qquad
 \E_{J_r,O_r\mid h_r}\langle p_r,\widehat\ell_r^2\rangle
 \le d_r^2\sum_{i\in\mathcal I_r}\frac1{q_r(i)}.
 \label{eq:moments}
\end{equation}
For a predictable comparator coordinate $v_r(h_r)$, let $e_r=\langle p_r,\widehat\ell_r\rangle-\widehat\ell_r(v_r)-\lref_r(J_r)+\lref_r(v_r)$. Under the recursively generated bounded stopped law $\mathbb P_{\rm stop}$,
\begin{align}
 \E_{\mathbb P_{\rm stop}}\sum_{r<N}\mathbf1\{r\text{ active}\}e_r&=0,
 \label{eq:stoppedfirst}\\
 \E_{\mathbb P_{\rm stop}}\sum_{r<N}\mathbf1\{r\text{ active}\}\langle p_r,\widehat\ell_r^2\rangle
 &\le\sum_{r<N}\E_{\mathbb P_{\rm stop}^{(r)}}
   \left[\mathbf1\{r\text{ active}\}d_r^2\sum_iq_r(i)^{-1}\right].
 \label{eq:stoppedsecond}
\end{align}
Here $\mathbb P_{\rm stop}^{(r)}$ is the actual completed-prefix law, not a distribution obtained by revealing the stopping time in advance.
\end{lemma}

\paragraph{From the potential to expected selection regret.}
A reference experiment supplies~\eqref{eq:recursion} and its integrable moments on one coupling $\Gamma_\mu$. For a supported predictable path $\mathbf v=\mathbf v^\mu=(v_0,\ldots,v_R)$ with $v_r\in\mathcal I_r$, define
\begin{equation}
 C(\mathbf v)=-\log\nu_0(v_0)-\sum_{r<R}\log \Theta_r(v_r,v_{r+1}).
 \label{eq:code}
\end{equation}
Require $C(\mathbf v)\le C$ on every modeled path, with deterministic $C$, and the deterministic second-moment budget
\begin{equation}
 \sum_{r=0}^R\E_{\Gamma_\mu}\!\left[\sum_{i\in\mathcal I_r}p_r(i)\widehat\ell_r(i)^2\right]\le V_{\rm est}.
 \label{eq:capacity}
\end{equation}
This is a raw second moment, not a centered variance. The coupling must induce the same pre-draw law as the observation interface. Path support alone does not ensure predictability.

\begin{theorem}[Growing selection]
\label{thm:selection}
Under these conditions,
\begin{equation}
 \sum_{r=0}^{R}\E_{\Gamma_\mu}[\lref_r(J_r)-\lref_r(v_r)]
 \le \frac C\eta+\frac{\eta V_{\rm est}}{2}.
 \label{eq:selection}
\end{equation}
\end{theorem}
Theorem~\ref{thm:selection} balances path complexity $C$ against feedback uncertainty $V_{\rm est}$ through $\eta$. Its potential proof uses Lemma~\ref{lem:stopped}, with the observation interface and recurrence sharing the same pre-draw law. Equality of isolated identity marginals is insufficient. No-regret behavior is derived when these terms grow slowly enough.

\section{Counterfactual regret guarantees}
\label{sec:transfer}
We first decompose regret, then bound deployment and generation effects locally. This order identifies which discrepancies each theorem controls.
\subsection{Regret decomposition and transfer}
Fix $\mu$ and a coupling $\Gamma_\mu$ preserving both native laws. Every endpoint below is an expected total at fixed horizon $T$, whose subscript is suppressed. In particular, $\Jsel=\sum_{r=0}^R\E_{\Gamma_\mu}\lref_r(J_r)$ and $\Jav=\sum_{r=0}^R\E_{\Gamma_\mu}\lref_r(v_r)$, whereas $\Jna=J_T(A)$ and $\Jnm=J_T(\mu)$. Specify intermediate score random variables and integrate them under this same coupling. The resulting certificates are
\begin{equation}
\begin{aligned}
 \Jna-\Jra&\le E_A, &\Jra-\Jsel&\le F+\mathcal P_b+D_A,\\
 \Jav-\Jal&\le G, &\Jal-\Jcl&\le bS,\\
 \Jcl-\Jpa&\le \hresp(S+1)+D_C, &\Jpa-\Jrm&\le E_{\rm nf},\\
 \Jrm-\Jnm&\le E_C, &V_{\rm est}&\le bV,
\end{aligned}
\label{eq:certificates}
\end{equation}
where $\mathcal P_b=\hresp(\lceil T/b\rceil+N_b+1)$. The integers $b\ge1$ and $N_b,S,\hresp\ge0$ are the nominal run cap, extra-boundary envelope, path switch bound, and response scale. A \emph{certificate} here is an explicit bound such as~\eqref{eq:certificates}, assumed at an interface or derived from local conditions, not an empirical validation. The response scale is not an automatic mixing time. The deterministic $V$ bounds scaled second-moment capacity, not a Bellman value.

The learner reference score is $\Jra$. The available substitute has score $\Jav$. The score $\Jal$ evaluates the ideal represented comparator path after alignment to run starts, on clean reference slots. The score $\Jcl$ removes run alignment, $\Jpa$ removes the reference response discrepancy, and $\Jrm$ is the reference meta-policy endpoint after undoing representation, before native deployment. These names specify the order of the comparison. Each endpoint must be defined before a certificate is asserted.

$F$ charges displaced foreground task opportunities. $\mathcal P_b$ charges clean-prefix loss due to at most $N_b$ extra boundaries. $D_A,D_C$ are score-calibration allowances. $G$ charges policies missing from the available library. $S$ bounds represented path changes. $E_{\rm nf}$ charges the loss of advantage under representation. $E_A,E_C$ charge implementation differences. Hard resource costs constrain admissibility separately. A birth is not automatically both a displaced task and a unit task penalty.

\begin{theorem}[Conditional native counterfactual transfer]
\label{thm:native}
Let $\mathcal M$ be nonempty. Suppose every $\mu\in\mathcal M$ has a compatible reference experiment, own-law endpoint realizations, and certificates~\eqref{eq:certificates} with common budgets, $C$, and $\eta>0$. Put $\Delta=E_A+F+D_A+G+D_C+E_{\rm nf}+E_C$. Then
\begin{equation}
 R_T(A,\mathcal M)\le U_T
 =\Delta+\frac C\eta+\frac{\eta bV}{2}
     +\hresp(\lceil T/b\rceil+N_b+1)+bS+\hresp(S+1).
 \label{eq:native}
\end{equation}
\end{theorem}
Theorem~\ref{thm:native} transfers a reference selection bound by accounting for each adjacent discrepancy once. The native learner need not execute Fixed Share, but complete-execution identities must be justified. Comparator couplings may differ while preserving the same learner law. Unrelated logs do not establish this correspondence.

\begin{corollary}[Conditional no regret]
\label{cor:native}
Along a positive-horizon family, all budgets, $b,N_b,S,\hresp,C,V,\eta$ may depend on $T$ but are uniform over $\mu$ at that horizon. If each of the five normalized terms
\begin{equation}
 \frac\Delta T,\quad\frac{C/\eta+\eta bV/2}{T},\quad
 \frac{\hresp(\lceil T/b\rceil+N_b+1)}T,\quad\frac{bS}T,\quad\frac{\hresp(S+1)}T
 \label{eq:conditions}
\end{equation}
tends to zero, then $\max\{0,R_T(A,\mathcal M)/T\}\to0$.
\end{corollary}
This is upper no regret, not convergence of signed regret to zero. For fixed $\hresp>0$ and fixed $b$, the third term has limit at least $\hresp/b>0$. Thus fixed run caps do not automatically give sublinearity. Appendix~\ref{app:notation} illustrates compatible horizon scalings.

\subsection{Resource-sensitive deployment stability}
\label{sec:local}
Mean-preserving rounding need not preserve nonlinear reflection, so we compare full-state transition kernels. Here $S_t$ ranges over a time-varying finite type containing all behavior-relevant memory, card-library, resource, and runtime components. Use common initial law $p^{\rm init}$ and scores $c_t(S_t)\in[0,1]$. An update affects the next state, not the score already produced at boundary $t$. These are state-score positions. Appendix~\ref{app:v19-connections} proves the completed-task padding with $T+1$ positions for $T$ tasks and zero terminal resource, preserving both boundary scores.

Let $K_t$ be a native kernel and $L_t$ its feedback-processed reference. Define $J_T(K)=\E_{\mathbb P_K}\sum_{t<T}c_t(S_t)$. Assume
\begin{equation}
 \TV(K_t(s,\cdot),L_t(s,\cdot))
 =\tfrac12\sum_j|K_t(s,j)-L_t(s,j)|\le\varepsilon_t(s).
 \label{eq:tv}
\end{equation}
For each remaining horizon $n$ with $t+1+n\le T$, let $V^L_{t+1,n}$ be its expected accumulated score. Whenever $\varepsilon_t(s)>0$, assume its range over next states is contained in an interval of length $h_t(s)\ge0$. At zero-error states no such range hypothesis is needed.

\begin{theorem}[Local stability on the actual native path]
\label{thm:stability}
The preceding local conditions imply
\begin{equation}
 |J_T(K)-J_T(L)|\le\E_{\mathbb P_K}\sum_{t<T}h_t(S_t)\varepsilon_t(S_t).
 \label{eq:localcharge}
\end{equation}
For nonnegative $\bar H,\bar\varepsilon,\chi_t(s),r_t(s)$, assume $h_t(s)\le\bar H$, $\varepsilon_t(s)\le\bar\varepsilon\,\chi_t(s)$, and $c_{\min}\chi_t(s)\le r_t(s)$ for $c_{\min}>0$. If every positive-probability native path has $\sum_{t<T}r_t(S_t)\le B_T$, then
\begin{equation}
 |J_T(K)-J_T(L)|\le\bar H\bar\varepsilon B_T/c_{\min}.
 \label{eq:resourcegap}
\end{equation}
Alternatively, without a bounded-influence assumption, a uniform local TV bound $\varepsilon_t(s)\le\bar\varepsilon_T$ gives $|J_T(K)-J_T(L)|\le T^2\bar\varepsilon_T$.
\end{theorem}
Theorem~\ref{thm:stability} weights each local shift by its downstream sensitivity along the actual law. Appendix~\ref{app:joint-budget} explains its role in joint resource design.

\begin{theorem}[Resource-sensitive native transfer]
\label{thm:repair}
Suppose Theorem~\ref{thm:native} applies to the reference endpoints $L_A,L_\mu$, with completed-task scores connected by Appendix~\ref{app:v19-connections}. Apply the local resource contract of Theorem~\ref{thm:stability} to $K_A,L_A$ and uniformly to $K_\mu,L_\mu$, with budgets $B_A,B_C$ and common $\bar H,\bar\varepsilon,c_{\min}$. Then
\begin{equation}
 J_T(K_A)-\inf_\mu J_T(K_\mu)
 \le U_T+\frac{\bar H\bar\varepsilon}{c_{\min}}(B_A+B_C).
 \label{eq:repair}
\end{equation}
For fixed $\bar H\bar\varepsilon/c_{\min}$, $U_T=o(T)$ and $B_A+B_C=o(T)$ imply upper no regret. The alternative local-only bound is sublinear if $T\bar\varepsilon_T\to0$ on both endpoints.
\end{theorem}
Theorem~\ref{thm:repair} restores both native endpoints, hence the two resource budgets. These outer connections must not duplicate implementation allowances already comparing the same endpoints in~\eqref{eq:certificates}.

\subsection{Generation reliability and the main guarantee}
\label{sec:localgeneration}
Theorem~\ref{thm:repair} can be instantiated without assuming that the total generation gap is small. We call the following joint success and quality requirements \emph{local generation reliability}. Fix each comparator before play. Decompose its available-to-aligned gap into finitely many causal generation episodes. In episode $e$, a normalized finite kernel generates trial outcomes using previous trial history. A Boolean analysis predicate identifies a validated and retained comparator-suitable outcome. On every still-unsuccessful history, its next-trial probability is at least $\rho_e\in(0,1]$. The algorithm need not know which card is optimal. This is a conditional coverage assumption, not an oracle input or an independence assumption.

After $k_e$ trials, let $\xi_e\in\{0,1\}$ indicate that none succeeded under its actual episode law $\mathbb P_e$. There are $\tau_e$ delay tasks and $n_e$ post-generation tasks. On each completed generation history, require the conditional expected per-task available-minus-aligned gap to be at most $1$ during delay and at most $\delta_e+\xi_e$ afterwards, with $\delta_e\ge0$. The successful card must support a causal available path satisfying this local quality condition. A useful text that is discarded or inaccessible is not a success. The episode sums must equal the actual generation endpoint difference, without counting the same task twice.

\begin{theorem}[Native guarantee from local conditions]
\label{thm:local}
Write $g_T^{\rm gen}(\mu)=\Jav-\Jal$. Under these episode conditions, uniformly for the specified comparator class,
\begin{equation}
 g_T^{\rm gen}(\mu)\le\sum_e\{\tau_e+n_e[\delta_e+(1-\rho_e)^{k_e}]\}.
 \label{eq:generatedbudget}
\end{equation}
For nonnegative envelopes, if $\sum_e\tau_e\le A_T$, $\sum_en_e\le T$, $\delta_e\le\delta_T$, and $(1-\rho_e)^{k_e}\le f_T$, the right side is at most $A_T+T(\delta_T+f_T)$.
Keep the other reference certificates, use identity deployment arrows inside the reference transfer, and impose the full-state local resource contracts on the two outer native-to-reference arrows. Use the validated completed-task instance and padded stability clock of Appendix~\ref{app:v19-connections}. With $U_T^{\rm rest}$ denoting~\eqref{eq:native} with $G=E_A=E_C=0$, this gives
\begin{equation}
 R_T\le U_T^{\rm rest}+A_T+T(\delta_T+f_T)
       +\kappa(B_A+B_C),\qquad
 \kappa=\bar H\bar\varepsilon/c_{\min}.
 \label{eq:fullylocal}
\end{equation}
For a positive-horizon family with fixed $\rho\in(0,1]$ and $\kappa$, suppose $f_T=(1-\rho)^{k_T}$, $k_T\to\infty$, $\delta_T\to0$, and $U_T^{\rm rest},A_T,B_A+B_C=o(T)$. Then $\max\{0,R_T/T\}\to0$.
\end{theorem}

\paragraph{A concrete native instance.}
Appendix~\ref{app:concrete} constructs a native verified-template reflector satisfying the interfaces jointly. Appendix~\ref{app:imperfect-growth} extends the instance evidence to repeated growth with imperfect validation and persistent selection errors.
\begin{corollary}[Native verified-template reflection]
\label{cor:instance}
For each integer $k\ge1$, the constructed instance has $T=k^2+k$ tasks, hard budget $2k+2$, and own-history comparator class $\mathcal M_0$ containing immediate repair and no reflection. Then
\begin{equation}
 R_T(A_k,\mathcal M_0)
 \le (2+\log2)k+k^2 2^{-k}+\tfrac14
 \le (3+\log2)k+\tfrac14.
 \label{eq:concrete-bound}
\end{equation}
Positive normalized regret vanishes along this family, whereas no reflection has regret $T$.
\end{corollary}

\begin{corollary}[Quality and hard-budget tradeoff]
\label{cor:joint-budget}
Fix a causal design satisfying the episode conditions and remaining interfaces of Theorem~\ref{thm:native}, uniformly over hard-feasible comparators. At tolerance $\delta_e\ge0$ and coverage $\rho_e(\delta_e)\in(0,1]$, let $c_e(\delta_e)\ge0$ bound prepaid trial cost including failures, $k_e$ bound prefix counts, and $B_T^{\rm other}\ge0$ bound other costs. Put $\overline U_T=U_T^{\rm rest}+E_A+E_C$. If~\eqref{eq:joint-hard} holds, the learner is hard-feasible and satisfies~\eqref{eq:joint-native}, where
\begin{align}
 B_T^{\rm other}+\sum_e k_ec_e(\delta_e)&\le B_T,
 \label{eq:joint-hard}\\
 R_T&\le\overline U_T+\sum_e\left\{\tau_e+n_e\left[\delta_e+
 e^{-k_e\rho_e(\delta_e)}\right]\right\}.
 \label{eq:joint-native}
\end{align}
For a positive-horizon family, if $\sum_en_e\le T$, $\delta_e\le\delta_T\to0$, $k_e\rho_e(\delta_e)\ge\lambda_T\to\infty$, and $\overline U_T+\sum_e\tau_e=o(T)$, then $\max\{0,R_T/T\}\to0$.
\end{corollary}
Appendix~\ref{app:joint-budget} gives feasible quotas and compatibility with Corollary~\ref{cor:instance}, charging delay and deployment once.

\section{Near-matching bounds under paid growth}
\label{sec:benchmark}
\subsection{The paid-growth model}
This theoretical benchmark isolates paid public growth and sparse feedback. It is not an empirical evaluation or a lower bound for every native interface. There are $m\ge1$ public phases of $n\ge1$ primitive rounds, with $T=mn$ and integer registration budget $B\ge4m$. A public registration schedule releases a strictly growing library. With zero-based phase index $j<m$, set
\begin{equation}
 K_j=\lceil B/2\rceil+
 \left\lfloor\frac{(j+1)\lfloor B/2\rfloor}{m}\right\rfloor.
 \label{eq:schedule}
\end{equation}
Each registration costs one unit, the total is $B$, and all intermediate prefixes are hard-feasible. Old identities persist, unborn identities are illegal, and the schedule is public.

Put $j(t)=\lfloor t/n\rfloor$ and $[K]_0=\{0,\ldots,K-1\}$. The legal table domain is $\mathcal D=\{(t,i)\mid 0\le t<T,\ i\in[K_{j(t)}]_0\}$. An oblivious table $\ell\in\{0,1\}^{\mathcal D}$ has entries $\ell_t^{\rm base}(i)$. The learner emits $I_t\in[K_{j(t)}]_0$. Only its selected bit is observed, independently erased with probability $1-q$, where $0<q\le1$. Task loss is
\begin{equation}
 Y_t=\tfrac12\bigl(\ell^{\rm base}_t(I_t)+\mathbf1\{t>0,\ I_t\ne I_{t-1}\}\bigr).
 \label{eq:task}
\end{equation}
The comparator path class is $\mathcal V=\prod_{j=0}^{m-1}[K_j]_0$. For $\mathbf v\in\mathcal V$, define $I_t^{\mathbf v}=v_{j(t)}$ and $L_T(\mathbf v,\ell)$ by summing~\eqref{eq:task} with $I_t=I_t^{\mathbf v}$. Thus both sides pay actual switches, the first draw is free, and registration is not charged again.

Let $\mathfrak A$ contain all causal randomized phase controllers, using all previously visible phase and primitive history. Define
\begin{equation}
 \mathcal R^*_{\rm bin}(T,m,B,q)
 =\inf_{A\in\mathfrak A}\sup_{\ell\in\{0,1\}^{\mathcal D}}
 \left\{\E_{\mathbb P_{A,\ell,q}}\sum_{t<T}Y_t-\min_{\mathbf v\in\mathcal V}L_T(\mathbf v,\ell)\right\},
 \label{eq:minimax}
\end{equation}
The law $\mathbb P_{A,\ell,q}$ covers learner and logging randomness for fixed $\ell$. Both bounds use this same model and cost.

\begin{theorem}[Paid sparse-feedback benchmark]
\label{thm:benchmark}
Suppose additionally that
\begin{equation}
 q^3T\le mB\log B.
 \label{eq:sparse}
\end{equation}
Then
\begin{equation}
 \frac1{128}\min\!\left\{T,\sqrt{\frac{mBT}{q}}\right\}-\frac m2
 \le\mathcal R^*_{\rm bin}
 \le\min\!\left\{T,\,8\sqrt{\frac{mB\log B\,T}{q}}+5m\right\}.
 \label{eq:benchmark}
\end{equation}
The libraries grow strictly in every phase, total registration cost is $B$, and every intermediate registration prefix is hard-feasible.
\end{theorem}
\subsection{Proof outline and sublinear regimes}
\label{sec:benchmarkproof}
The lower construction hides a good identity per phase. The upper learner counts actual switches. Appendix~\ref{app:benchmark} proves matching up to logarithmic and boundary terms in this restricted regime.

\begin{corollary}[Learnable sparse-feedback families]
\label{cor:benchmark}
For benchmark instances satisfying~\eqref{eq:sparse}, if
\begin{equation}
 \frac{mB\log B}{qT}\longrightarrow0,
 \qquad \frac mT\longrightarrow0,
 \label{eq:benchmarkrate}
\end{equation}
then $\max\{0,\mathcal R^*_{\rm bin}/T\}\to0$.
\end{corollary}
These conditions force $q\to0$ and require effective observations $qT$ to dominate identification complexity. The benchmark and published all-$q$ switching bounds~\citep{dekel2014switch} have distinct scopes, detailed in Appendix~\ref{app:benchmark}.

\section{Lean formalization}
\label{sec:verification}
Our formal contribution is a compositional verification of native reflection, from probabilistic execution and local contracts to counterfactual regret and sublinear limits. Table~\ref{tab:proof-chain} identifies its verified connections.

\paragraph{Why composition requires verification.}
Valid component bounds yield a guarantee for the same agent only with compatible trajectory laws, scoring clocks and resource conventions. Our proofs verify marginal preservation, completed-task alignment and adaptive observation moments, establishing that the combined bound concerns the specified native execution.

\begin{table}[H]
\centering
\small
\renewcommand{\arraystretch}{1.12}
\caption{Proof obligations addressed by the Lean formalization.}
\label{tab:proof-chain}
\begin{tabularx}{\linewidth}{@{}p{0.24\linewidth}X p{0.18\linewidth}@{}}
\toprule \textbf{Proof obligation} & \textbf{Purpose of formal checking} & \textbf{Formal result}\\\midrule
Probability coherence & Preserve native laws across couplings. & Prop.~\ref{prop:wellposed}, Thm.~\ref{thm:native}\\
Adaptive observations & Derive moments under bounded stopping. & Lem.~\ref{lem:stopped}, Thm.~\ref{thm:selection}\\
Generation reliability & Couple quality, trials and hard costs. & Thm.~\ref{thm:local}, Cor.~\ref{cor:joint-budget}\\
Deployment stability & Control error propagation on actual paths. & Thms.~\ref{thm:stability}--\ref{thm:repair}\\
Composition and limits & Align task clocks and limit hypotheses. & Thm.~\ref{thm:local}, Cor.~\ref{cor:native}\\
Benchmark consistency & Match cost and feedback conventions. & Thm.~\ref{thm:benchmark}\\
Concrete instantiation & Derive interfaces from native execution. & Cor.~\ref{cor:instance}, Prop.~\ref{prop:imperfect-growth}\\
\bottomrule
\end{tabularx}
\end{table}

\paragraph{Reusable guarantees and scope.}
Concrete native instances derive compatible contracts from generation, validation and execution. Explicit interfaces identify which obligations must be re-established when mechanisms change, supporting reuse of the theory. All thirteen results have Lean proofs~\citep{lean2021,mathlib2020} of the \emph{end-to-end conditional mathematical implications}. Applications must establish their contracts. Appendix~\ref{app:map} supplies declaration mappings and build provenance for independent checking.

\section{Conclusion}
\label{sec:conclusion}
We connect local generation reliability and deployment stability to conditional native-agent regret, providing a unified view of how reflective mechanisms can affect long-term performance under explicit resource constraints. Concrete native instances demonstrate sublinear regret under hard budgets, repeated growth and imperfect validation. Lean verifies the composed proof chains, complemented by near-matching paid-growth bounds.
The resulting framework separates algorithmic guarantees from application-specific assumptions through explicit interfaces, making clear which contracts must be established when generation, validation, or execution mechanisms change. This perspective enables reflective agents to be analyzed through modular guarantees, where improvements in individual components can be incorporated without changing the overall comparison framework. The results establish a foundation for reasoning about adaptive agents whose capabilities evolve through generation, validation, and execution.
\label{main:end}

\clearpage

\bibliography{iclr2027_conference}

@inproceedings{mourtada2017,
 title={Efficient Tracking of a Growing Number of Experts},
 author={Mourtada, Jaouad and Maillard, Odalric-Ambrym},
 booktitle={Proceedings of the 28th International Conference on Algorithmic Learning Theory},
 volume={76}, pages={517--539}, year={2017},
 url={https://proceedings.mlr.press/v76/mourtada17a.html}
}

@inproceedings{rouyer2021,
 title={An Algorithm for Stochastic and Adversarial Bandits with Switching Costs},
 author={Rouyer, Chlo{\'e} and Seldin, Yevgeny and Cesa-Bianchi, Nicol{\`o}},
 booktitle={Proceedings of the 38th International Conference on Machine Learning},
 volume={139}, pages={9127--9135}, year={2021},
 url={https://proceedings.mlr.press/v139/rouyer21a.html}
}

@inproceedings{dudik2011,
 title={Doubly Robust Policy Evaluation and Learning},
 author={Dud{\'i}k, Miroslav and Langford, John and Li, Lihong},
 booktitle={Proceedings of the 28th International Conference on Machine Learning},
 year={2011},
 url={https://icml.cc/2011/papers/554_icmlpaper.pdf}
}

@inproceedings{alon2015,
 title={Online Learning with Feedback Graphs: Beyond Bandits},
 author={Alon, Noga and Cesa-Bianchi, Nicol{\`o} and Dekel, Ofer and Koren, Tomer},
 booktitle={Proceedings of the 28th Conference on Learning Theory},
 volume={40}, pages={23--35}, year={2015},
 url={https://proceedings.mlr.press/v40/Alon15.html}
}

@book{lattimore2020,
 title={Bandit Algorithms},
 author={Lattimore, Tor and Szepesv{\'a}ri, Csaba},
 publisher={Cambridge University Press}, year={2020},
 url={https://tor-lattimore.com/downloads/book/book.pdf}
}

@article{bubeck2012,
 title={Regret Analysis of Stochastic and Nonstochastic Multi-armed Bandit Problems},
 author={Bubeck, S{\'e}bastien and Cesa-Bianchi, Nicol{\`o}},
 journal={Foundations and Trends in Machine Learning},
 volume={5}, number={1}, pages={1--122}, year={2012},
 url={https://arxiv.org/abs/1204.5721}
}

@inproceedings{lean2021,
 title={The {Lean 4} Theorem Prover and Programming Language},
 author={de Moura, Leonardo and Ullrich, Sebastian},
 booktitle={Automated Deduction---CADE 28},
 pages={625--635}, year={2021},
 url={https://lean-lang.org/papers/lean4.pdf}
}

@inproceedings{mathlib2020,
 title={The {Lean} Mathematical Library},
 author={{The mathlib Community}},
 booktitle={Proceedings of the 9th ACM SIGPLAN International Conference on Certified Programs and Proofs},
 year={2020},
 doi={10.1145/3372885.3373824},
 url={https://arxiv.org/abs/1910.09336}
}

@inproceedings{cheng2023feedback,
  title={Understanding the Role of Feedback in Online Learning with Switching Costs},
  author={Cheng, Duo and Zhou, Xingyu and Ji, Bo},
  booktitle={Proceedings of the 40th International Conference on Machine Learning},
  pages={5521--5543},
  year={2023},
  url={https://proceedings.mlr.press/v202/cheng23f.html}
}

@book{cesabianchi2006,
  title     = {Prediction, Learning, and Games},
  author    = {Cesa-Bianchi, Nicol\`o and Lugosi, G\'abor},
  publisher = {Cambridge University Press},
  year      = {2006},
  
  url={https://cesa-bianchi.di.unimi.it/predbook/}
}

@article{auer2002,
  title   = {The Nonstochastic Multiarmed Bandit Problem},
  author  = {Auer, Peter and Cesa-Bianchi, Nicol\`o and Freund, Yoav and Schapire, Robert E.},
  journal = {SIAM Journal on Computing},
  volume  = {32},
  number  = {1},
  pages   = {48--77},
  year    = {2002},
  
  url={https://epubs.siam.org/doi/10.1137/S0097539701398375}
}

@article{herbster1998,
  title   = {Tracking the Best Expert},
  author  = {Herbster, Mark and Warmuth, Manfred K.},
  journal = {Machine Learning},
  volume  = {32},
  number  = {2},
  pages   = {151--178},
  year    = {1998},
  
  url={https://link.springer.com/article/10.1023/A:1007424614876}
}

@inproceedings{arora2012,
  title     = {Online Bandit Learning against an Adaptive Adversary: From Regret to Policy Regret},
  author    = {Arora, Raman and Dekel, Ofer and Tewari, Ambuj},
  booktitle = {Proceedings of the 29th International Conference on Machine Learning},
  year      = {2012},
  
  url={https://arxiv.org/abs/1206.6400}
}

@inproceedings{cesabianchi2013,
  title     = {Online Learning with Switching Costs and Other Adaptive Adversaries},
  author    = {Cesa-Bianchi, Nicol\`o and Dekel, Ofer and Shamir, Ohad},
  booktitle = {Advances in Neural Information Processing Systems},
  volume    = {26},
  year      = {2013},
  
  url={https://proceedings.neurips.cc/paper/2013/hash/9cf81d8026a9018052c429cc4e56739b-Abstract.html}
}

@inproceedings{dekel2014switch,
  title     = {Bandits with Switching Costs: {$T^{2/3}$} Regret},
  author    = {Dekel, Ofer and Ding, Jian and Koren, Tomer and Peres, Yuval},
  booktitle = {Proceedings of the 46th Annual ACM Symposium on Theory of Computing},
  pages     = {459--467},
  year      = {2014},
  doi={10.1145/2591796.2591868},
  url={https://arxiv.org/abs/1310.2997}
}

@inproceedings{anava2015,
  title     = {Online Learning for Adversaries with Memory: Price of Past Mistakes},
  author    = {Anava, Oren and Hazan, Elad and Mannor, Shie},
  booktitle = {Advances in Neural Information Processing Systems},
  volume    = {28},
  year      = {2015},
  
  url={https://proceedings.neurips.cc/paper/2015/hash/38913e1d6a7b94cb0f55994f679f5956-Abstract.html}
}

@inproceedings{daniely2015strongly,
  title     = {Strongly Adaptive Online Learning},
  author    = {Daniely, Amit and Gonen, Alon and Shalev-Shwartz, Shai},
  booktitle = {Proceedings of the 32nd International Conference on Machine Learning},
  pages     = {1405--1411},
  year      = {2015},
  volume={37},
  url={https://proceedings.mlr.press/v37/daniely15.html}
}

@inproceedings{saha2020,
  title     = {Improved Sleeping Bandits with Stochastic Action Sets and Adversarial Rewards},
  author    = {Saha, Aadirupa and Gaillard, Pierre and Valko, Michal},
  booktitle = {Proceedings of the 37th International Conference on Machine Learning},
  pages     = {8357--8366},
  year      = {2020},
  volume={119},
  url={https://proceedings.mlr.press/v119/saha20a.html}
}

@article{ghalme2020,
  title   = {Ballooning Multi-Armed Bandits},
  author  = {Ghalme, Ganesh and Dhamal, Swapnil and Jain, Shweta and Gujar, Sujit and Narahari, Yadati},
  journal = {arXiv preprint arXiv:2001.10055},
  year    = {2020},
  
  url={https://arxiv.org/abs/2001.10055}
}

@inproceedings{nguyen2024,
  title     = {Near-Optimal Per-Action Regret Bounds for Sleeping Bandits},
  author    = {Nguyen, Quan M. and Mehta, Nishant},
  booktitle = {Proceedings of the 27th International Conference on Artificial Intelligence and Statistics},
  year      = {2024},
  volume={238}, pages={2827--2835},
  url={https://proceedings.mlr.press/v238/nguyen24d.html}
}

@inproceedings{park2025,
  title     = {Do {LLM} Agents Have Regret? A Case Study in Online Learning and Games},
  author    = {Park, Chanwoo and Liu, Xiangyu and Ozdaglar, Asuman and Zhang, Kaiqing},
  booktitle = {International Conference on Learning Representations},
  year      = {2025},
  url       = {https://proceedings.iclr.cc/paper_files/paper/2025/hash/65ea878cb90b440e8b4cd34fe0959914-Abstract-Conference.html}
}

@inproceedings{nguyentang2025,
  title     = {Policy-Regret Minimization in Markov Games with Function Approximation},
  author    = {Nguyen-Tang, Thanh and Arora, Raman},
  booktitle = {Proceedings of the 42nd International Conference on Machine Learning},
  year      = {2025},
  volume={267}, pages={46242--46264},
  url={https://proceedings.mlr.press/v267/nguyen-tang25a.html}
}

@inproceedings{xu2026gas,
  title     = {Online Decision Making with Generative Action Sets},
  author    = {Xu, Jianyu and Jain, Vidhi and Wilder, Bryan and Singh, Aarti},
  booktitle = {International Conference on Learning Representations},
  year      = {2026},
  url       = {https://proceedings.iclr.cc/paper_files/paper/2026/hash/07ed3b22786bd8d46a3654d5325bca09-Abstract-Conference.html}
}

@inproceedings{shinn2023reflexion,
  title     = {Reflexion: Language Agents with Verbal Reinforcement Learning},
  author    = {Shinn, Noah and Cassano, Federico and Gopinath, Ashwin and Narasimhan, Karthik and Yao, Shunyu},
  booktitle = {Advances in Neural Information Processing Systems},
  volume    = {36},
  year      = {2023},
  
  url={https://proceedings.neurips.cc/paper_files/paper/2023/hash/1b44b878bb782e6954cd888628510e90-Abstract-Conference.html}
}

@inproceedings{madaan2023selfrefine,
  title     = {Self-Refine: Iterative Refinement with Self-Feedback},
  author    = {Madaan, Aman and Tandon, Niket and Gupta, Prakhar and Hallinan, Skyler and Gao, Luyu and Wiegreffe, Sarah and Alon, Uri and Dziri, Nouha and Prabhumoye, Shrimai and Yang, Yiming and Gupta, Shashank and Majumder, Bodhisattwa Prasad and Hermann, Katherine and Welleck, Sean and Yazdanbakhsh, Amir and Clark, Peter},
  booktitle = {Advances in Neural Information Processing Systems},
  volume    = {36},
  year      = {2023},
  
  url={https://proceedings.neurips.cc/paper_files/paper/2023/hash/91edff07232fb1b55a505a9e9f6c0ff3-Abstract-Conference.html}
}

@inproceedings{merten2018,
 title={Verified Learning Without Regret: From Algorithmic Game Theory to Distributed Systems with Mechanized Complexity Guarantees},
 author={Merten, Samuel and Bagnall, Alexander and Stewart, Gordon},
 booktitle={Programming Languages and Systems (ESOP 2018)},
 series={Lecture Notes in Computer Science}, volume={10801},
 pages={561--588}, year={2018},
 doi={10.1007/978-3-319-89884-1_20},
 url={https://link.springer.com/chapter/10.1007/978-3-319-89884-1_20}
}
\bibliographystyle{iclr2027_conference}


\clearpage
\begin{center}
{\LARGE \bfseries Appendix}
\end{center}


\appendix
\section{Own-history executions and their realizations}
\label{app:native}

\subsection{Probability objects and proof of Proposition~4}
Write $\mathcal H_t^{\rm task}$ for the type of completed prefixes and $Q_t^A(h,\cdot)$ for the next-record kernel of $A$. Starting from its prescribed initial law $\mathbb P_A^{(0)}$, define the prefix law recursively by kernel integration. For a nonnegative measurable test function $f$ this means
\begin{equation}
 \int f(h,x)\,d\mathbb P_A^{(t+1)}(h,x)
 =\int_{\mathcal H_t^{\rm task}}\left[\int f(h,x)Q_t^A(h,dx)\right]d\mathbb P_A^{(t)}(h).
 \label{eq:v14-law}
\end{equation}
Taking $f=1$ and using $Q_t^A(h,\mathcal Z_{t+1})=1$ proves inductively that every prefix law is a probability measure. For a finite event model the inner integrals are ordinary finite sums. For the native measurable model they are the kernel integrals used in \texttt{NativeSemantics}. The comparator construction replaces $Q^A$ by $Q^\mu$ everywhere, including its history argument.

\paragraph{Measurability and integrability.}
Let $\mathcal H_t^{\rm task}$ denote the finite prefix of the product record space. A normalized measurable control kernel from $\mathcal H_t^{\rm task}$ to $\mathcal Z_{t+1}$ recursively induces a trajectory probability law. In the companion, this is the causal trajectory construction used by \texttt{NativeSemantics.Execution.law}, not an unrelated measure supplied after the fact. Each score is a measurable function of a completed prefix. Composing with the measurable prefix projection makes $Y_t$ measurable. Finite summation makes $L_T$ measurable. Summing the unit bounds gives $0\le L_T\le T$. Since the trajectory law is a probability measure, the constant $T$ is integrable. Domination proves integrability of $L_T$. Monotonicity of integration gives $0\le J_T\le T$.

For a measurable realization $r_A$, the pushforward identity means that its distribution is the actual causal law. Integration under a measurable pushforward gives
\begin{equation}
 \int L_T^A(r_A(\omega))\,d\Gamma(\omega)
 =\int L_T^A(x)\,d((r_A)_\#\Gamma)(x)
 =\int L_T^A(x)\,d\mathbb P_A(x).
 \label{eq:realizationproof}
\end{equation}
The preceding integrability argument justifies these real-valued integrals. The formal result is \texttt{NativeSemantics.Realization.expected\_total}.

\paragraph{Finite loss bounds.}
For each trajectory $x$, the unit-score contract gives
\begin{equation}
 0\le\sum_{t=0}^{T-1}Y_t(x)\le\sum_{t=0}^{T-1}1=T.
 \label{eq:v14-unit}
\end{equation}
Each summand is measurable, hence the finite sum is measurable. The integrable dominating constant has integral $T$ because the law has total mass one. In particular, the comparator loss set is nonempty and bounded below by zero. Its real infimum is therefore the intended mathematical infimum, rather than a default value for an empty or unbounded set.

\paragraph{Why the comparator is outside the expectation.}
Fix a causal procedure $\mu$ before evaluating its expectation. Its initial conditions and kernels determine its own law $\mathbb P_\mu$. A coupling can correlate exogenous randomness across executions for a proof, but must preserve both marginals. It cannot give $\mu$ the library generated by $A$ unless that library would also be generated in the comparator execution. In particular, an infimum over realized paths chosen after all random outcomes is not the infimum in~\eqref{eq:objective}.

\paragraph{Resource modes.}
The main theorem uses the hard class directly. If a mean class is considered, integrability of prefix costs is an additional prerequisite for ordinary real expectations. The checked inclusion theorem integrates each hard prefix inequality under the own law of the candidate. The hard-class optimum can exceed the mean-class optimum, so changing the comparator class requires changing the objective. It is not a modification of the implementation charge. These ancillary definitions and identities are in \texttt{ResourceClasses}. The revised main theorem does not assume the hard and mean classes coincide.

\section{Observation and stopping proofs}
\label{app:stopped}

\subsection{Local moments in Lemma~\ref{lem:stopped}}
For this proof only, suppress the slot and history indices. Let $I,O$ be finite types, $p(i)>0$ with $\sum_i p(i)=1$, and $k(i,o)\ge0$ with $\sum_o k(i,o)=1$. Let $M(i,o)\in\{0,1\}$ and $y(i,o)\in[0,d]$. Define
\begin{equation}
 \widehat\ell_{j,o}(i)=\frac{\mathbf1\{j=i\}M(j,o)y(j,o)}{p(i)q(i)},
 \qquad \mathbb P(J=j,O=o\mid h)=p(j)k(j,o).
 \label{eq:v14-local-event}
\end{equation}
All expectations in this subsection refer to this conditional joint law. In particular, the symbols in the estimator are evaluated at the observed event $(j,o)$, whereas its coordinate is $i$.

\paragraph{First moment.}
Fix a completed history and suppress its index. The joint event probability of draw $j$ and outcome $o$ is $p(j)k(j,o)$. Its sum is one. At coordinate $i$, all summands with $j\ne i$ are zero. Hence
\begin{equation}
 \sum_{j,o}p(j)k(j,o)\widehat\ell_{j,o}(i)
 =\frac{p(i)}{p(i)q(i)}\sum_o k(i,o)M(i,o)y(i,o)
 =\lref(i).
 \label{eq:firstproof}
\end{equation}
The division is legitimate because $p(i),q(i)>0$. There is no hidden claim that a missing observation is itself a zero task loss. It contributes zero to the estimator and is corrected by the denominator.

\paragraph{Second moment.}
The accepted score satisfies $(My)^2\le M d^2$, because $M\in\{0,1\}$ and $0\le y\le d$. Consequently
\begin{equation}
 \E_{J,O\mid h}\widehat\ell(i)^2
 \le\frac{p(i)q(i)d^2}{p(i)^2q(i)^2}
 =\frac{d^2}{p(i)q(i)}.
 \label{eq:secondproof}
\end{equation}
Multiplying by $p(i)$ and summing proves the second inequality in~\eqref{eq:moments}. Finite sums also give equality between the expected estimated learner--comparator difference and its true loss difference. The comparator coordinate must be fixed at this history. A coordinate selected using the current draw need not satisfy the same equality.

\paragraph{The learner--comparator moment identity.}
Let $v$ be determined by the completed history, before $(J,O)$ is drawn. Finite linearity and~\eqref{eq:firstproof} give
\begin{align}
 \E_{J,O\mid h}\!\left[\sum_i p(i)\widehat\ell_{J,O}(i)-\widehat\ell_{J,O}(v)\right]
 &=\sum_i p(i)\lref(i)-\lref(v),\label{eq:v14-linear1}\\
 \E_{J,O\mid h}[\lref(J)-\lref(v)]
 &=\sum_i p(i)\lref(i)-\lref(v).\label{eq:v14-linear2}
\end{align}
Subtracting the two equalities yields zero conditional estimation noise. No independence of the whole past and future is used. Predictability of $v$ is essential because it permits applying the coordinate identity at the fixed coordinate $v$.

\subsection{The generated stopped law}
Let $E_r$ be the finite event type at slot $r$. Complete histories are built recursively from $\operatorname{Option}(E_r)$. The terminal value \texttt{none} is absorbing. At an active history the next-event kernel is the normalized local observation kernel on \texttt{some} events. At an inactive history it assigns probability one to \texttt{none}. The Boolean activity rule depends only on the completed history. Let $\varpi_r(h)$ be the resulting prefix mass. Its recursion is
\begin{equation}
 \varpi_0(())=1,\qquad \varpi_{r+1}(h,x)=\varpi_r(h)k_r^{\rm stop}(h,x).
 \label{eq:stoplaw}
\end{equation}
Summing over $x$ proves normalization by induction. Nonnegativity follows by multiplying nonnegative local masses. These are actual path probabilities, not importance weights.

For any cumulative event score $F_r$, expand one step of the expectation, sum over the last event, and use kernel normalization on the old score. This gives
\begin{equation}
 \E_{\mathbb P^{(r+1)}_{\rm stop}}F_{r+1}
 =\E_{\mathbb P^{(r)}_{\rm stop}}F_r+
   \sum_h \varpi_r(h)\sum_xk_r^{\rm stop}(h,x)f_r(h,x).
 \label{eq:stopstep}
\end{equation}
For the estimation noise, the last sum is zero at inactive histories by definition and zero at active histories by~\eqref{eq:firstproof}. Induction proves~\eqref{eq:stoppedfirst}. For the quadratic score, retain the actual local expectation instead of replacing it by zero, then apply~\eqref{eq:secondproof}. Summing the inequalities with nonnegative prefix masses proves~\eqref{eq:stoppedsecond}. These final implications are \texttt{StoppedObservation.Protocol.stopped\_unbiased} and \texttt{stopped\_second}. They invoke the normalized recursive stopped process, rather than assuming a stopped martingale identity.

\paragraph{Detailed induction on the stopped prefix.}
Set $F_0=0$ and $F_{r+1}(h,x)=F_r(h)+f_r(h,x)$, with $f_r=0$ on the terminal event. Expanding the finite product law gives
\begin{align}
 \sum_{h,x}\varpi_r(h)k_r^{\rm stop}(h,x)F_{r+1}(h,x)
 &=\sum_h\varpi_r(h)F_r(h)\underbrace{\sum_xk_r^{\rm stop}(h,x)}_{=1}\notag\\
 &\quad+\sum_h\varpi_r(h)\sum_xk_r^{\rm stop}(h,x)f_r(h,x).
 \label{eq:v14-stop-expand}
\end{align}
For estimation noise, the second line vanishes at every history. The expectation remains zero at $r+1$ if it is zero at $r$, establishing the first stopped identity from the base case. For the quadratic score, it is bounded by
\begin{equation}
 \sum_h\varpi_r(h)\mathbf1\{h\text{ is active}\}\,d_r(h)^2
       \sum_{i\in\mathcal I_r}q_r(h,i)^{-1}.
 \label{eq:v14-stop-var}
\end{equation}
Summing these one-step inequalities for $r=0,\ldots,N-1$ proves the cumulative capacity bound. Nonnegativity of $\varpi_r$ preserves every inequality. Bounded $N$ makes all sums finite, so no unbounded optional-stopping theorem is being silently applied.

\paragraph{Scope of the local modification.}
The observation kernel must represent the run that is actually being scored. If the chosen identity itself determines how many task rounds are included, inserting its original draw probability into a common-duration estimator is not justified by this proof. The revised statement retains the synchronized accepted-score contract at this precise boundary. The activity rule for the next completed slot remains fully history-adapted.

\section{Potential and native-transfer proofs}
\label{app:potential}

\subsection{Proof of Theorem~\ref{thm:selection}}
\paragraph{Step 1. Positivity and conservation under sharing.}
The initial total mass is $W_0=1$. If $W_r>0$, at least one $w_r(i)$ is positive. Since $\exp(-\eta \widehat\ell_r(i))>0$, the tilted mass is positive. Summing the update over $j$ and using $\sum_j\Theta_r(i,j)=1$ proves that sharing preserves this tilted total. Thus $W_{r+1}>0$ by induction and all logarithms and probabilities below are defined.

Fix a realized sequence of finite types, transitions, and nonnegative estimates. Let $W_r=\sum_i w_r(i)$. Row stochasticity in~\eqref{eq:recursion} yields
\begin{equation}
 W_{r+1}=W_r\sum_i p_r(i)e^{-\eta \widehat\ell_r(i)}.
 \label{eq:normalizer}
\end{equation}
The exponential is strictly positive. Starting from $W_0=1$, every $W_r$ is positive. For $x\ge0$, $e^{-x}\le1-x+x^2/2$, and $\log u\le u-1$ for $u>0$. Therefore
\begin{equation}
 \log W_{r+1}-\log W_r
 \le-\eta\langle p_r,\widehat\ell_r\rangle
       +\frac{\eta^2}{2}\langle p_r,\widehat\ell_r^2\rangle.
 \label{eq:potentialstep}
\end{equation}
No upper bound on the importance-weighted estimate is used here. This explains why arbitrarily small positive probabilities are compatible with the expectation proof, though not with an automatic high-probability statement.

\paragraph{Step 2. The logarithmic upper bound.}
For $x\ge0$, $e^{-x}\le1-x+x^2/2$. Applying this coordinatewise and summing with probability weights gives
\begin{align}
 \frac{W_{r+1}}{W_r}
 &\le1-\eta\sum_i p_r(i)\widehat\ell_r(i)
          +\frac{\eta^2}{2}\sum_i p_r(i)\widehat\ell_r(i)^2,\label{eq:v14-ratio}\\
 \log W_{R+1}
 &\le-\eta\sum_{r=0}^R\langle p_r,\widehat\ell_r\rangle
          +\frac{\eta^2}{2}\sum_{r=0}^R\langle p_r,\widehat\ell_r^2\rangle.
 \label{eq:v14-log-upper}
\end{align}
The first right side is positive because it bounds the strictly positive ratio from above. Applying $\log u\le u-1$ to the ratio gives the one-step logarithmic inequality. The second display is its telescoping sum, using $\log W_0=0$.

\paragraph{Step 3. A supported path lower bound.}
Define its incoming mass at slot $r$ by
\begin{equation}
 \omega_r^{\mathbf v}=\nu_0(v_0)\prod_{s<r}\Theta_s(v_s,v_{s+1})
               \exp\!\left(-\eta\sum_{s<r}\widehat\ell_s(v_s)\right).
 \label{eq:v14-path-weight}
\end{equation}
Induction proves $w_r(v_r)\ge \omega_r^{\mathbf v}$. The base case is equality. For the induction step, retain the single nonnegative summand indexed by $v_r$ in the update to $w_{r+1}(v_{r+1})$. Multiplication by the positive transition and exponential preserves the inequality. At the final slot, retain the mass at $v_R$ and sum its outgoing row. This is why the final displayed bound uses $R$ transitions but $R+1$ scored slots.

The supported path contributes a nonnegative term to the final total mass. Following it through slots $0,\ldots,R$ gives
\begin{equation}
 W_{R+1}\ge\nu_0(v_0)
       \prod_{r<R}\Theta_r(v_r,v_{r+1})
       \exp\!\left(-\eta\sum_{r=0}^R\widehat\ell_r(v_r)\right).
 \label{eq:pathmass}
\end{equation}
The last update is summed over its destination row, so there is no missing comparator transition after slot $R$. Taking logs, summing~\eqref{eq:potentialstep}, and rearranging gives the deterministic path inequality
\begin{equation}
 \sum_{r=0}^R\bigl(\langle p_r,\widehat\ell_r\rangle-\widehat\ell_r(v_r)\bigr)
 \le\frac{C(\mathbf v)}\eta+\frac\eta2\sum_{r=0}^R\langle p_r,\widehat\ell_r^2\rangle.
 \label{eq:pathpotential}
\end{equation}

\paragraph{Step 4. Combine the two sides and take expectation.}
The lower logarithm is $-C(\mathbf v)-\eta\sum_{r=0}^R\widehat\ell_r(v_r)$. Comparing it with~\eqref{eq:v14-log-upper}, moving the loss terms to one side, and dividing by $\eta>0$ proves~\eqref{eq:pathpotential}. This argument is pointwise for every realized collection of finite types and stochastic transitions.

The proof is \texttt{Growing.growing\_path\_regret}. Integrate on the stipulated comparator coupling. The integrability assumptions permit finite summation and integration to commute. The moment identities identify the left side with actual reference selection differences. The code and quadratic budgets prove~\eqref{eq:selection}, as in \texttt{DependentSelection.expected\_growing\_selection}.

\paragraph{Compatibility with the observation theorem.}
To use the local calculation in the expected path theorem, the probabilities $p_r$ in the event kernels must be those produced by the weight recurrence, and the score difference must be the same random variable. For history-dependent type families, these identifications and integrability are the explicit \texttt{Experiment} contract. The finite stopped-observation result proves moments for its stated event types. It does not construct every possible native observation model. Thus a supplied moment identity is a visible interface premise, not a claimed consequence of unrelated logs.

\subsection{Proof of Theorem~\ref{thm:native}}
For one fixed comparator, insert every endpoint in its stated order.
\begin{equation}
 \begin{aligned}
 \Jna-\Jnm={}&(\Jna-\Jra)+(\Jra-\Jsel)+(\Jsel-\Jav)+(\Jav-\Jal)\\
 &+(\Jal-\Jcl)+(\Jcl-\Jpa)+(\Jpa-\Jrm)+(\Jrm-\Jnm).
 \end{aligned}
 \label{eq:telescope}
\end{equation}
The selection arrow $\Jsel-\Jav$ is bounded by Theorem~\ref{thm:selection}. Substitute all seven directed certificates and $V_{\rm est}\le bV$. This gives~\eqref{eq:native} for a fixed comparator. Proposition~\ref{prop:wellposed} identifies the realized native endpoint expectations with their own-law expected losses. The formal result permits comparator-specific coupling measures and type families.

\paragraph{Budget substitution without hidden sign assumptions.}
Substitution into~\eqref{eq:telescope} first gives
\begin{align}
 J_T(A)-J_T(\mu)
 &\le E_A+(F+\mathcal P_b+D_A)+C/\eta+\eta V_{\rm est}/2+G\notag\\
 &\quad+bS+\hresp(S+1)+D_C+E_{\rm nf}+E_C.\label{eq:v14-before-budget}
\end{align}
Since $\eta>0$ and $V_{\rm est}\le bV$, replace only the quadratic term by $\eta bV/2$. Substitute $\mathcal P_b=\hresp(\lceil T/b\rceil+N_b+1)$ and group the seven scalar charges into $\Delta$. This produces exactly $U_T$ in~\eqref{eq:native}. Each adjacent difference is present once. No absolute values or independently maximized versions of individual endpoints are inserted.

Write the uniform upper bound as $U_T$. For every $\mu$, the fixed-comparator result states $J_T(A)-U_T\le J_T(\mu)$. Thus $J_T(A)-U_T$ is a lower bound on the nonempty set of comparator expected losses. The defining greatest-lower-bound property implies $J_T(A)-U_T\le\inf_\mu J_T(\mu)$, proving the theorem. There is no exchange of expectation and infimum.

\paragraph{Proof of Corollary~\ref{cor:native}.}
Divide the definition of $U_T$ by its positive horizon. Its normalized value is exactly the sum in~\eqref{eq:conditions}. Finite addition preserves limits, so $U_T/T\to0$. The main bound implies
\begin{equation}
 0\le\max\{0,R_T/T\}\le\max\{0,U_T/T\}.
 \label{eq:squeeze}
\end{equation}
Continuity of the maximum and squeezing prove convergence. This is checked both as a numerical lemma and as a theorem on families of native executions with interface witnesses.

\section{Local-feedback stability and resource charging}
\label{app:feedback}
\paragraph{Finite paths and Bellman values.}
For a kernel sequence $K$, define $V^K_{t,0}(s)=0$ and
\begin{equation}
 V^K_{t,n+1}(s)=c_t(s)+\sum_jK_t(s,j)V^K_{t+1,n}(j).
 \label{eq:bellman}
\end{equation}
For an explicit next-state trace, its probability is the product of the traversed kernel probabilities. Inductively summing over the last or first next state proves normalization. Expanding the trace loss and using this normalization proves that the trace expectation equals~\eqref{eq:bellman}. Averaging with the common initial distribution gives $J_T(K)$. These equalities are in \texttt{FeedbackTrajectory}, so the Bellman quantity is not merely a proxy with an unproved relation to the objective.

\paragraph{Range-sensitive TV inequality.}
If $f(j)\in[a,a+h]$ and $p,q$ are probability distributions, their signed mass difference sums to zero. Subtract the midpoint $a+h/2$. Then
\begin{equation}
 \left|\sum_j(p(j)-q(j))f(j)\right|
 =\left|\sum_j(p(j)-q(j))(f(j)-a-h/2)\right|
 \le\tfrac h2\sum_j|p(j)-q(j)|=h\TV(p,q).
 \label{eq:rangeTV}
\end{equation}
For zero TV the distributions agree on the finite set, so the expectation difference is zero without imposing a range condition.

\subsection{Proof of Theorem~\ref{thm:stability}}
\paragraph{Step 1. Exact one-step difference.}
For a remaining horizon $n+1$, add and subtract the same reference continuation under $K_t$.
\begin{align}
 V^K_{t,n+1}(s)-V^L_{t,n+1}(s)
 &=\sum_jK_t(s,j)\{V^K_{t+1,n}(j)-V^L_{t+1,n}(j)\}\notag\\
 &\quad+\sum_j\{K_t(s,j)-L_t(s,j)\}V^L_{t+1,n}(j).
 \label{eq:v14-bellman-exact}
\end{align}
The common immediate score cancels. Taking absolute values uses the triangle inequality and $K_t(s,j)\ge0$. The second sum is bounded by the range-sensitive TV inequality. If the local TV is zero, the distributions agree and this term is zero without a range condition.

Subtract the two Bellman equations and add and subtract the $K$-average of the reference continuation. The immediate scores cancel because they are the same function of the complete pre-update state. Writing $\zeta_t(s)=h_t(s)\varepsilon_t(s)$ gives
\begin{equation}
 |V^K_{t,n+1}(s)-V^L_{t,n+1}(s)|
 \le\sum_jK_t(s,j)|V^K_{t+1,n}(j)-V^L_{t+1,n}(j)|+\zeta_t(s).
 \label{eq:bellmandifference}
\end{equation}
The second term follows from~\eqref{eq:rangeTV} and the local contract. Induction on $n$, starting at zero continuation, bounds the right side by the $K$-Bellman value of the accumulated charge $\zeta$. Averaging at time zero and using the trace interpretation proves~\eqref{eq:localcharge}.

\paragraph{Step 2. Propagate local charges along the native law.}
Define $\mathcal D_{t,0}(s)=0$ and
\begin{equation}
 \mathcal D_{t,n+1}(s)=\zeta_t(s)+\sum_jK_t(s,j)\mathcal D_{t+1,n}(j).
 \label{eq:v14-charge-bellman}
\end{equation}
At $n=0$, both value differences are zero. If $|V^K_{t+1,n}-V^L_{t+1,n}|\le \mathcal D_{t+1,n}$, nonnegativity of $K_t$ and~\eqref{eq:bellmandifference} imply the same bound at $(t,n+1)$. Therefore $|V^K_{0,T}-V^L_{0,T}|\le \mathcal D_{0,T}$. Averaging over the common initial distribution $p^{\rm init}$ yields
\begin{align}
 |J_T(K)-J_T(L)|
 &\le\sum_sp^{\rm init}(s)|V^K_{0,T}(s)-V^L_{0,T}(s)|\notag\\
 &\le\sum_sp^{\rm init}(s)\mathcal D_{0,T}(s)
 =\E_{\mathbb P_K}\sum_{t<T}\zeta_t(S_t).
 \label{eq:v14-charge-path}
\end{align}
The final equality is the finite trajectory/Bellman identity, applied to the nonnegative score $\zeta$ instead of the task score $c$. It fixes the law as $\mathbb P_K$, not a worst-case union of possible update times.

\paragraph{Step 3. The hard-resource inequality.}
For the hard-resource version, the pointwise hypotheses yield
\begin{equation}
 0\le \zeta_t(s)\le\bar H\bar\varepsilon \chi_t(s)
 \le\frac{\bar H\bar\varepsilon}{c_{\min}}r_t(s).
 \label{eq:chargecost}
\end{equation}
Monotonicity of the $K$-Bellman recursion transfers this inequality to path expectations. For every positive-mass trace the cumulative resource is at most $B_T$. Zero-mass traces contribute zero, irrespective of their values. Normalization therefore gives $\E_{\mathbb P_K}\sum_t r_t(S_t)\le B_T$. Substitute in~\eqref{eq:localcharge} to prove~\eqref{eq:resourcegap}. No worst-case union over possibly active times is introduced.

For the local-only alternative, induction in~\eqref{eq:bellman} gives $0\le V^L_{t,n}\le n$. At boundary $t$, the remaining continuation has at most $T-t-1$ scores. Hence
\begin{equation}
 |J_T(K)-J_T(L)|\le\sum_{t<T}(T-t-1)\bar\varepsilon_T
 \le T^2\bar\varepsilon_T.
 \label{eq:localonly}
\end{equation}
This argument does not assume mixing or forgetting. Its precision condition is correspondingly stronger.

\paragraph{Proof of Theorem~\ref{thm:repair}.}
For each fixed comparator, insert the two reference expectations.
\begin{equation}
 J_T(K_A)-J_T(K_\mu)
 =[J_T(K_A)-J_T(L_A)]+[J_T(L_A)-J_T(L_\mu)]
   +[J_T(L_\mu)-J_T(K_\mu)].
 \label{eq:repairproof}
\end{equation}
Bound the first and last terms by~\eqref{eq:resourcegap}, and invoke the proved reference transfer for the middle term. Uniformity in $\mu$ permits the same infimum argument used for Theorem~\ref{thm:native}. In Lean, \texttt{NativeRepairEndpoint.counterfactual} calls \texttt{ModelMain.main\_fixed}, together with actual kernel-path bounds. It does not accept the desired reference regret conclusion as a new axiom. Its resource-side structure contains local TV and influence bounds, local cost domination, and a hard path budget.

For no regret, normalize~\eqref{eq:repair}. The sum of $U_T/T$, a fixed constant times $B_A/T$, and a fixed constant times $B_C/T$ tends to zero. Apply~\eqref{eq:squeeze}. For~\eqref{eq:localonly}, division by $T$ gives $T\bar\varepsilon_T$, whose assumed convergence proves the alternative.

\section{Local generation and the combined native theorem}
\label{app:localgeneration}
\paragraph{Finite causal trial laws.}
Fix episode $e$ and suppress $e$ on its trial quantities. Let $\mathcal H_{e,k}^{\rm gen}$ be the finite type of completed trial outcomes in an episode. A normalized nonnegative kernel $g_k(h,i)$ selects the next outcome. Write $\mathbb P_e^{(k)}$ for the completed-trial law, with mass $\varpi_k^{\rm gen}(h)$. Its mass recursion satisfies $\varpi_0^{\rm gen}(())=1$ and $\varpi_{k+1}^{\rm gen}(h,i)=\varpi_k^{\rm gen}(h)g_k(h,i)$. The outcome can encode a candidate, validation, and retention. A success indicator $b_k(h,i)\in\{0,1\}$ marks an outcome whose retained card supports the stipulated causal available path. It may depend on the comparator fixed for the proof and on completed history, but is not an oracle supplied to the generator.

Let $\xi_0=1$ and $\xi_{k+1}(h,i)=\xi_k(h)(1-b_k(h,i))$. On still-unsuccessful histories assume
\begin{equation}
 \sum_i g_k(h,i)(1-b_k(h,i))\le1-\rho,\qquad 0<\rho\le1.
 \label{eq:hazard}
\end{equation}
Expand one step of the actual law.
\begin{align}
 \mathbb P_e(\xi_{k+1}=1)
 &=\sum_h \varpi_k^{\rm gen}(h)\xi_k(h)\sum_i g_k(h,i)(1-b_k(h,i))\\
 &\le(1-\rho)\sum_h\varpi_k^{\rm gen}(h)\xi_k(h)
 =(1-\rho)\mathbb P_e(\xi_k=1).
 \label{eq:failurestep}
\end{align}
Successful-history terms vanish. All remaining path masses are nonnegative. Induction from probability one gives $\mathbb P_e(\xi_k=1)\le(1-\rho)^k$, without trial independence. The checked result is \texttt{LocalGenerationCertificate.failure\_bound}.

\paragraph{Local quality and the semantic endpoint.}
For completed trial history $h$, let $g^{\rm pre}_t(h)$ and $g^{\rm post}_t(h)$ be conditional expected task differences between a specified causal available path and its aligned ideal path. Their exact finite-law expectation must equal the corresponding part of $\Jav-\Jal$. This is a semantic identification, not a bound on the generation opportunity gap. An application must derive it from the score laws and a disjoint episode decomposition.

Assume $g^{\rm pre}_t(h)\le1$ on $\tau$ delay tasks and $g^{\rm post}_t(h)\le\delta+\xi_k(h)$ on $n$ later tasks. Then
\begin{align}
 \E_{\mathbb P_e^{(k)}}\left[\sum_{t<\tau}g^{\rm pre}_t+\sum_{t<n}g^{\rm post}_t\right]
 &\le\E_{\mathbb P_e^{(k)}}[\tau+n\delta+n\xi_k]\\
 &\le\tau+n[\delta+(1-\rho)^k].
 \label{eq:episodegap}
\end{align}
The first inequality sums local bounds with nonnegative masses. Normalization evaluates constants. The second uses the proved failure bound. This is \texttt{Episode.gap\_le}, not a field assuming a small total gap.

Episodes need not be independent, but their laws must describe their actual contributions. Random entry histories require uniformly valid conditional contracts and the correct marginal or conditional representation. Resetting history in the proof while keeping its effects in the deployment is invalid. Sum~\eqref{eq:episodegap} to obtain~\eqref{eq:generatedbudget}. Nonnegative envelopes and $\sum_en_e\le T$ give $A_T+T(\delta_T+f_T)$ by finite summation and monotonicity.

\paragraph{Uniform envelopes without separating an infimum.}
The comparison is first proved for a single fixed $\mu$, with its own episode laws. For each episode, $n_e\ge0$ and the envelope bounds imply
\begin{equation}
 \tau_e+n_e[\delta_e+(1-\rho_e)^{k_e}]
 \le\tau_e+n_e(\delta_T+f_T).
 \label{eq:v14-generation-envelope}
\end{equation}
Summing, using $\sum_e\tau_e\le A_T$, $\sum_en_e\le T$, and $\delta_T+f_T\ge0$ gives
\begin{equation}
 g_T^{\rm gen}(\mu)\le A_T+(\delta_T+f_T)\sum_en_e\le A_T+T(\delta_T+f_T).
 \label{eq:v14-generation-total}
\end{equation}
Uniformity means these same numeric envelopes work for every comparator, although its episode laws may differ. There is no step that independently optimizes one envelope for one comparator and another envelope for a different comparator.

\paragraph{The combined endpoint.}
The formal structure \texttt{Remainder} retains acquisition, alignment, response, representation, and capacity conditions. It has no generation or deployment bound. The constructor \texttt{Remainder.certify} derives generation from episodes, identifies the inner deployment endpoints, and assigns zero to their budgets. The theorem \texttt{native\_local} invokes the actual reference theorem with this constructed witness and the two local kernel-resource contracts. Thus the outer native kernel comparison includes any selector discrepancy, not just changes to displayed feedback.

The validated \texttt{CompletedNativeV19.Instance.bound} carries the kernels, reference experiment, episodes, integrable score endpoints, and completed-task identities. Its generation budget is defined by delay, quality, and failure envelopes, while its inner reference deployment arrows are identities. Its conclusion is the original own-history regret. No bound on that regret or on either derived total gap is an input. Other reference interfaces remain explicit. Subtracting the generation budget from the reference bound leaves exactly $U_T^{\rm rest}$, proving~\eqref{eq:fullylocal} without duplicate deployment charges.

\paragraph{Exact budget identity.}
Let $U_T$ be the reference upper bound with its generation budget $G=A_T+T(\delta_T+f_T)$ and internal deployment budgets zero. By definition $U_T^{\rm rest}=U_T-G$. Consequently
\begin{align}
 R_T&\le U_T+\kappa B_A+\kappa B_C\notag\\
 &=U_T^{\rm rest}+A_T+T(\delta_T+f_T)+\kappa(B_A+B_C).
 \label{eq:v14-native-budget}
\end{align}
The first line is an application of the proved transfer, with the generation certificate just constructed. The second is an exact scalar identity. No adjacent endpoint difference may be charged twice. The same causal mechanism can affect different adjacent differences, which must both be retained. An outer full-state charge already covers propagation of its own perturbation. Appendix~\ref{app:v19-connections} gives an explicit accounting example.

\paragraph{The family no-regret theorem.}
Each horizon supplies a complete instance of the stated model and contracts. Its finite state domains, trial-outcome domains, episode index set, comparator coupling laws, and reference identity domains may vary with the horizon. Uniformity of a budget is over comparators at that horizon, not over all horizons. The formal family uses a common ambient measurable space and comparator index type. These do not require identical kernels or identical finite domains across horizons.

Divide~\eqref{eq:fullylocal} by $T>0$. With fixed $\rho\in(0,1]$ and $\kappa$, its normalized upper bound is
\begin{equation}
 \frac{U_T^{\rm rest}}T+\frac{A_T}T+\delta_T+(1-\rho)^{k_T}
       +\kappa\frac{B_A(T)+B_C(T)}T.
 \label{eq:v18-family-normalized}
\end{equation}
Because $0\le1-\rho<1$ and $k_T\to\infty$, the geometric term tends to zero. Each remaining term tends to zero under the hypotheses of Theorem~\ref{thm:local}. This argument uses the combined resource condition directly, without adding separate resource limits. The positive part is continuous, so~\eqref{eq:squeeze} proves the claimed upper no-regret limit.

The publication endpoint is \texttt{CompletedNativeV19.variable\_family}. For each family index it invokes the validated completed-task bound, then combines the geometric and normalized limits with denominator $T$, not $T+1$. It does not take the desired regret bound as a premise. The earlier \texttt{LocalNativeFamily.no\_regret} remains a fixed-domain family result with separately normalized resource assumptions. The earlier correspondence endpoint remains a state-score family result. The new wrapper supplies the completed-task, variable-domain and combined-resource correspondence without changing~\eqref{eq:fullylocal}.

\paragraph{Strength and cost of the assumptions.}
The available path must be causal, not chosen retrospectively from favorable trial outcomes. Retention and comparator-relative quality must persist on the episode being charged. If increasingly precise good cards become rarer, their smaller hazards must be reflected in $f_T$. A fixed positive hazard cannot be presumed for many mutually exclusive targets. Attempts also consume actual generation and validation resources. Increasing their number is legitimate only when delay and hard budgets permit it. These restrict the admissible native model class, not the original architecture. They are not proved properties of every deployed reflector.

\section{The finite paid lower and upper benchmark}
\label{app:benchmark}
The sparse-family conditions force $q\to0$, not unlearnability at fixed $q$. This benchmark does not lower-bound every native certificate. Published all-$q$ switching constructions~\citep{dekel2014switch} address a broader regime.
\subsection{Information lower bound for the full causal class}
\paragraph{Finite information construction.}
In a phase with $K$ identities, choose a hidden target uniformly. Its independent binary base losses have mean $1/2-\epsilon$, while other identities have mean $1/2$. Let $\mathbb P_0$ be the all-$1/2$ observation law and $\mathbb P_j$ the target-$j$ law for the same causal randomized learner. The erasure symbol is distinct from a zero loss. All laws are constructed from finite kernels on visible histories.

The Bernoulli divergence and erasure calculation give, for $0\le\epsilon\le1/4$,
\begin{equation}
 \KL(\operatorname{Ber}(1/2)\Vert\operatorname{Ber}(1/2-\epsilon))
 =-\tfrac12\log(1-4\epsilon^2)\le4\epsilon^2.
 \label{eq:bernoulli}
\end{equation}
Independent erasure multiplies this divergence by $q$. On each completed visible history the learner uses the same conditional action kernel under both environments. Its factor cancels in the supported finite chain rule. If $N_j$ counts draws of $j$, induction on history length yields the divergence bound $4q\epsilon^2\E_{\mathbb P_0}N_j$. Finite testing inequalities then give
\begin{equation}
 \E_{\mathbb P_j}N_j\le\E_{\mathbb P_0}N_j+n\sqrt{4q\epsilon^2\E_{\mathbb P_0}N_j}.
 \label{eq:countcompare}
\end{equation}
\paragraph{The divergence calculation in detail.}
For $u=4\epsilon^2\in[0,1/4]$, the elementary inequality $-\log(1-u)\le u/(1-u)$ implies
\begin{equation}
 -\tfrac12\log(1-4\epsilon^2)
 \le\frac{2\epsilon^2}{1-4\epsilon^2}
 \le\tfrac83\epsilon^2\le4\epsilon^2.
 \label{eq:v14-bernoulli-bound}
\end{equation}
Let $E_q(p)$ be the distribution on $\{\bot,0,1\}$ assigning mass $1-q$ to erasure and mass $qp(b)$ to bit $b$. The erasure contribution to relative entropy is zero and the factors $q$ cancel in each non-erased likelihood ratio. Thus
\begin{equation}
 \KL(E_q(p)\Vert E_q(p'))
 =\sum_{b\in\{0,1\}}qp(b)\log\frac{qp(b)}{qp'(b)}
 =q\KL(p\Vert p').
 \label{eq:v14-erasure-kl}
\end{equation}
At $q=1$ the zero-mass erasure term is omitted. Positive Bernoulli probabilities on both sides ensure the same observation support.

At a fixed visible history $h$, write $p_t^{\rm low}(i\mid h)$ for the learner action kernel. This same function is used under $\mathbb P_0$ and $\mathbb P_j$, even though the distributions of $h$ differ. On an action with positive mass, its factor cancels in the likelihood ratio. The next-event divergence is consequently
\begin{equation}
 p_t^{\rm low}(j\mid h)\,q\,
 \KL(\operatorname{Ber}(1/2)\Vert\operatorname{Ber}(1/2-\epsilon)).
 \label{eq:v14-one-event-kl}
\end{equation}
For all other actions the observation laws agree. Induction with the supported finite chain rule adds these conditional divergences under the null prefix law. Since $N_j=\sum_{t<n}\mathbf1\{I_t=j\}$,
\begin{align}
 \KL(\mathbb P_0\Vert \mathbb P_j)
 &\le4q\epsilon^2\sum_{t<n}\E_{\mathbb P_0^{(t)}}p_t^{\rm low}(j\mid H_t)\notag\\
 &=4q\epsilon^2\E_{\mathbb P_0}N_j.
 \label{eq:v14-chain-count}
\end{align}
Finally $N_j\in[0,n]$, so its expectation difference is at most $n\TV(\mathbb P_0,\mathbb P_j)$. The included finite testing inequality bounds this TV by $\sqrt{\KL(\mathbb P_0\Vert \mathbb P_j)}$, a conservative constant form. Combining with~\eqref{eq:v14-chain-count} proves~\eqref{eq:countcompare}. The proof handles zero-probability action histories on their support. It does not require the learner to assign positive probability to every action.

Write $x_j=\E_{\mathbb P_0}N_j$, $y_j=\E_{\mathbb P_j}N_j$. Since exactly one identity is selected per primitive round, $\sum_jx_j=n$. Cauchy--Schwarz bounds the sum of square roots in~\eqref{eq:countcompare}. Set
\begin{equation}
 \epsilon=\min\{1/4,\sqrt{K/(64qn)}\}.
 \label{eq:gapchoice}
\end{equation}
Then $64q\epsilon^2n\le K$, and the calculation gives $K^{-1}\sum_jy_j\le3n/4$. Expected regret against the hidden target is at least $\epsilon n/4$, which is at least $\min\{n,\sqrt{Kn/q}\}/32$. The scalar rate, the counts, and the underlying path distributions are all part of the proof chain, rather than a supplied information lower-bound certificate.

\paragraph{The count calculation and constants.}
Because $\sum_jx_j=n$, Cauchy--Schwarz gives $\sum_j\sqrt{x_j}\le\sqrt{Kn}$. Averaging~\eqref{eq:countcompare} therefore gives
\begin{equation}
 \frac1K\sum_jy_j\le\frac nK+
     n\sqrt{4q\epsilon^2n/K}\le\frac n2+\frac n4=\frac{3n}4.
 \label{eq:v14-count-average}
\end{equation}
The second inequality uses $K\ge2$ and $64q\epsilon^2n\le K$. Against hidden target $j$, expected base regret is $\epsilon(n-y_j)$, because each non-target draw loses $\epsilon$ in conditional mean. Its average is at least $\epsilon n/4$. Substitution of~\eqref{eq:gapchoice} gives
\begin{equation}
 \epsilon n/4=\min\{n/16,\sqrt{Kn/q}/32\}
       \ge\tfrac1{32}\min\{n,\sqrt{Kn/q}\}.
 \label{eq:v14-lower-constant}
\end{equation}
The finite KL/testing lemmas used before this calculation belong to the included proof closure. Classical references describe the general information method~\citep{lattimore2020,bubeck2012}, but no numerical lower bound from those references is inserted as an axiom.

\paragraph{Random tables and fixed instances.}
The loss table can be sampled before play. A finite factorization expresses the joint expectation as an expectation over complete oblivious tables followed by the learner and logging expectation conditional on each table. An average at least $L$ implies at least one fixed table has conditional expected regret at least $L$. Independent phase targets allow the argument to be conditioned on every preceding visible phase transcript. Thus it applies to the full causal phase-controller class, not just to independently restarted learners.

For the paid schedule, each eligible set has size between $B/2$ and $B$. Adding phases, applying the half-scale task definition~\eqref{eq:task}, and charging comparator boundary switches produces the left side of~\eqref{eq:benchmark}. The explicit constant $1/128$ is conservative. The proof uses the actual growing identity types in both the learner and the comparator, as recorded in \texttt{PaidInformationLower.paid\_minimax\_lower}.

\subsection{The primitive upper algorithm and its tuning}
\paragraph{Public construction.}
Within phase $j$, the public learner uses uniform initialization and a held-run cap and learning rate determined by $K_j,n,q$. Primitive execution records the identity emitted on every round. The proof identifies this primitive algorithm with its exponential-weights state recursion and counts the switches in the emitted identities. A partial final run is treated explicitly. No equality between a macro-policy and an unrelated primitive process is postulated.

For roundwise independent masks $M_t$, the run estimator for a fixed eligible identity is $\mathbf1\{J_r=i\}p_r(i)^{-1}\sum_{t\in\mathcal T_r}M_t\ell^{\rm base}_t(i)/q$. Expanding its square separates the diagonal terms from the cross-time terms. Independence gives a diagonal bound proportional to $u_r/q$ and a cross-time bound proportional to $u_r^2$. Applying exponential weights and charging actual held-run switches gives the phase bound. Phase composition, the actual boundary costs, and uniform library bounds give~\eqref{eq:generalupper}. These steps terminate at \texttt{PrimitiveJointBounds.uniform\_minimax\_upper}, in the same primitive minimax model as the lower bound.

\paragraph{Runwise variance, including cross-time terms.}
Here $\mathcal T_r$ is the consecutive task-index set in run $r$, $u_r=|\mathcal T_r|$, and $u$ denotes one such length. In the following local calculation $s=1,\ldots,u$ numbers positions inside the run, not internal operations. The vector $\mathbf M=(M_1,\ldots,M_u)$ contains its independent observation masks. The law $\mathbb P^{\rm phase}_{A,\ell,q}$ below is the induced phase law of the specified upper learner, with the loss table fixed. The symbols $L^{\rm task}_{\rm phase}$ and $L^{\rm base}_{\rm phase}(i)$ denote its cumulative task loss and the cumulative base loss of fixed identity $i$ in that phase.
Fix the oblivious table and the history at the start of a held run of length $u$. Conditional on $J=i$, let $X_i=\sum_{s=1}^u M_s\ell^{\rm base}_s(i)/q$. Independent Bernoulli masks give
\begin{align}
 \E_{\mathbf M\mid J=i,h}[X_i^2]
 &=\frac1q\sum_s\ell^{\rm base}_s(i)^2
       +2\sum_{s<t}\ell^{\rm base}_s(i)\ell^{\rm base}_t(i)\notag\\
 &\le u/q+u(u-1)\le u/q+u^2.
 \label{eq:v14-run-square}
\end{align}
For $\widehat\ell(i)=\mathbf1\{J=i\}X_i/p(i)$, it follows that
\begin{equation}
 \E_{J,M\mid h}\sum_i p(i)\widehat\ell(i)^2\le K(u/q+u^2).
 \label{eq:v14-run-capacity}
\end{equation}
For run lengths bounded by $b$ and totaling $n$, $\sum u^2\le bn$ and $\sum u=n$. Thus the phase capacity is $V_{\rm phase}=Kn(b+1/q)$. In particular, the whole squared run loss is not multiplied by $1/q$. Such a replacement would conflate one observation per run with independent per-round erasure.

\paragraph{Exact public tuning.}
Write $a=K\log K$ and set
\begin{equation}
 b=\min\{n,\max\{1,\lceil(n/a)^{1/3}\rceil\}\},\qquad
 V_{\rm phase}=Kn(b+1/q),\qquad \eta=\sqrt{2\log K/V_{\rm phase}}.
 \label{eq:v14-tuning}
\end{equation}
These are the \texttt{TunedPhase} parameters. They depend on public $n,K,q$, not on the optimal arm or future losses. Since $n\ge1$, $K\ge2$, and $q>0$, both $b$ and $\eta$ are positive. Holding an identity creates no within-run switch. Counting boundaries and the possible incomplete last run bounds internal switching by $n/b+1$. The checked primitive-to-run correspondence then yields
\begin{align}
 \E_{\mathbb P^{\rm phase}_{A,\ell,q}}[L^{\rm task}_{\rm phase}]-\tfrac12L^{\rm base}_{\rm phase}(i)
 &\le\tfrac12\{\log K/\eta+\eta V_{\rm phase}/2+n/b+1\}\notag\\
 &=\tfrac12\{\sqrt{2\log K\,V_{\rm phase}}+n/b+1\}.
 \label{eq:v14-phase-tuned}
\end{align}
The equality follows from $\eta^2V_{\rm phase}=2\log K$. Dividing by $2\eta>0$ gives $\log K/\eta=\eta V_{\rm phase}/2$, and both terms sum to $\eta V_{\rm phase}=\sqrt{2\log K\,V_{\rm phase}}$.

\paragraph{Integer caps and short phases.}
The formal tuning proof treats the cap cases separately. If $a\ge n$, the capped run size is $1$ and the phase task loss bound $n$ is already at most $a^{1/3}n^{2/3}$. If $a<n$, put $x=(n/a)^{1/3}>1$. Here $a\ge1$ and $n\ge1$ imply $x\le n$, so $x\le b\le x+1\le2x$. Consequently
\begin{equation}
 \sqrt{2anb}\le2a^{1/3}n^{2/3},\qquad
 n/b\le a^{1/3}n^{2/3}.
 \label{eq:v14-cap-bounds}
\end{equation}
Also $\sqrt{2\log K\,V_{\rm phase}}\le\sqrt{2anb}+\sqrt{2an/q}$. Combining these estimates, or using the trivial bound in the short-phase case, gives the conservative phase bound $4\{a^{1/3}n^{2/3}+\sqrt{an/q}+1\}$. The checked implementation includes its final partial run and actual emitted identities.

\paragraph{Phase composition and a single controller class.}
One additional boundary switch costs at most $1/2$ per phase. The upper controller resets its weights at the public phase boundaries, while the lower bound allows every causal controller to retain past history. Both controllers are nevertheless members of the same class $\mathfrak A$. Since $K_j\le B$ and $x\log x$ is increasing for $x\ge1$, summing the phase bounds and using $T=mn$ gives
\begin{equation}
 m(B\log B)^{1/3}n^{2/3}=(mB\log B)^{1/3}T^{2/3},\qquad
 m\sqrt{B\log B\,n/q}=\sqrt{mB\log B\,T/q}.
 \label{eq:v14-phase-scaling}
\end{equation}
The additive $4m+m/2$ is bounded by $5m$. Taking the supremum over fixed tables and then the infimum over controllers preserves this upper bound because the displayed public controller is admissible.

The preceding primitive argument gives the explicit bound
\begin{equation}
 \mathcal R^*_{\rm bin}\le4\left((mB\log B)^{1/3}T^{2/3}
 +\sqrt{mB\log B\,T/q}\right)+5m.
 \label{eq:generalupper}
\end{equation}
\paragraph{Proof of the sparse simplification.}
Let $x=a^{1/3}T^{2/3}$ and $y=\sqrt{aT/q}$, with $a,T,q>0$. Then $x^3=aT^2$ and $y^2=aT/q$. Multiplying $q^3T\le a$ by the nonnegative factor $a^2T^3$ shows $x^6\le y^6$. Both $x$ and $y$ are nonnegative, so monotonicity of the sixth power implies $x\le y$. Substitute into~\eqref{eq:generalupper}. Independently, cumulative task loss is at most $T$ and the comparator task loss is nonnegative. The minimax value is therefore at most $T$. These two upper bounds prove the minimum in~\eqref{eq:benchmark}.

The registration assertions are not inferred from the loss bound. They are separate arithmetic theorems about~\eqref{eq:schedule}. Every phase creates at least one identity under $B\ge4m$, the increments telescope to the total budget, and each intermediate registration prefix is bounded by the phase endpoint. Their conjunction with the minimax statements is checked in \texttt{PrimitiveJointBounds.paid\_common\_model}.

\paragraph{Proof of Corollary~\ref{cor:benchmark}.}
For positive $T,q$ and nonnegative $a$,
\begin{equation}
 \frac{\sqrt{aT/q}}T=\sqrt{\frac a{qT}}.
 \label{eq:sqrtnormalization}
\end{equation}
This follows by squaring both nonnegative sides, or by the square-root scaling identity. Take $a=mB\log B$. Continuity of the square root and~\eqref{eq:benchmarkrate} show that the normalized upper bound tends to zero. Apply~\eqref{eq:squeeze}. The proof in \texttt{SparseFeedbackBenchmark.sparse\_no\_regret} includes the normalization identity rather than treating asymptotic notation as an informal replacement for it.

\section{How to instantiate the native transfer}
\label{app:application}
This section describes obligations for general applications. Appendix~\ref{app:concrete} separately discharges them for a specified verified-template reflector. Neither section reports measurements of a deployed language model. Consider a frozen-weight coding agent whose completed record contains the task, generated program, test report, memory, versioned rule cards, and remaining validation budget. A fixed scorer maps the structured test report to a value in $[0,1]$. A rule card induces the execution policy $\pi_i$. The native controller $\mu_A$ is retained, rather than replaced by Algorithm~\ref{alg:reference}.

\paragraph{State and identity.}
The state must retain all components needed to predict future task scores. A card identifier names an immutable program and interpreter version, not its current probability of selection. If a finite-state local theorem is used, a finite alphabet and bounded storage must be specified, including the environment state needed for the Markov representation. Bounded prompt length alone does not make an unbounded external environment finite-state. An abstraction that discards state requires its own justified comparison, not an assertion that the finite theorem applies unchanged.

\paragraph{Reference observation interface.}
An evaluator must identify a synchronized score unit, the draw probability $p_r(i)$, and the acceptance probability $q_r(i)$. Randomly accepting complete test reports independently of their score is one way to meet the accepted-score identity. An evaluator that retains only successful tests is not. A rule card generated after an outcome is available for the next draw, not for the draw that generated that outcome. Endogenous termination of future slots is allowed, whereas changing the current score window after seeing a favorable outcome requires a separate analysis.

\paragraph{Local feedback interface.}
Specify a native raw-feedback kernel $K_t$ and a processed-feedback reference kernel $L_t$ on the same complete state space. The required quantity is their total variation, including the joint distribution of generated cards, memory, resources, and runtime outcomes. Preserving the mean of the displayed score does not imply this condition. Theorem~\ref{thm:stability} requires a uniform continuation-value range at perturbed states. Such a range needs a mixing, reset, or bounded-influence argument on the complete state. Deleting a card after a fixed number of rounds does not prove that argument. If no such property is available, only the explicitly weaker local-only bound is supplied here.

\paragraph{Resources, generation, and selection are different obligations.}
Count failed generation and validation attempts whenever their resource is consumed. A strictly positive minimum resource cost controls the number of active updates on each native path. Theorem~\ref{thm:local} derives the generation certificate from conditional coverage, per-task quality, and a causal available path. The number of registered cards alone proves none of these conditions. Finally, a native selection bridge must compare complete native and reference executions. An offline reference that receives logs from the native run is not automatically distributed as the online reference used in the potential argument. The learner law must stay fixed across comparator couplings.

\paragraph{Proved claims and required application evidence.}
The moment and potential bounds follow from the actual reference recursion and its stated local observation contract. The stability charge follows from complete-state TV, continuation range, and pathwise costs. The remaining directed certificates are assumptions to establish for the chosen model, preferably from local operational lemmas rather than from measured total regret. Finite observational data alone does not prove a uniform counterfactual certificate over all histories. A statistical validation theorem would need a specified sampling design, coverage assumptions, and uncertainty bounds. No such empirical certificate is claimed in this paper.

\paragraph{No duplicate charging.}
Choose the reference endpoints before evaluating any error. A local-feedback charge can replace the corresponding deployment certificate, or can compare a new outer pair of executions as in Theorem~\ref{thm:repair}. It cannot be added to an unchanged certificate for the same endpoint difference. Different adjacent differences may share a causal source. Resource feasibility restricts the comparator domain. Resource-induced displacement contributes to $F$ only if it displaces a task opportunity not already included in another endpoint. The paid benchmark registers births under its hard budget and charges actual switches in its task loss. Neither cost is an additional generic penalty in the native objective.

\section{Declaration-level verification map}
\label{app:map}
Definitions~\ref{def:history}--\ref{def:metapolicy} and the thirteen theoretical results use a single numbering sequence. Definitions fix semantic objects and are not themselves proof claims. Table~\ref{tab:map} records the checked endpoints. Intermediate declarations are reached through their Lean imports. Corollary~\ref{cor:instance} constructs a specified native model and derives its interfaces. Proposition~\ref{prop:imperfect-growth} adds repeated native growth and derives the persistent cost of false acceptance. Other applications must still establish their own contracts. The release includes a machine-readable mapping with source and build hashes.
\begin{table}[t]
\centering\footnotesize
\caption{Numbered results and their checked Lean endpoints. Prefix \texttt{DRGPR.} is omitted.}\label{tab:map}
\begin{tabular}{>{\raggedright\arraybackslash}p{0.19\textwidth}>{\raggedright\arraybackslash}p{0.75\textwidth}}
\toprule Result & Formal endpoint\\\midrule
Proposition~\ref{prop:wellposed} & \texttt{NativeSemantics.Execution.expected\_bounds}, \texttt{loss\_integrable}, and \texttt{Realization.expected\_total}\\
Lemma~\ref{lem:stopped} & \texttt{FiniteObservation.Round.estimate\_mean}, \texttt{weighted\_second}, and \texttt{StoppedObservation.Protocol.stopped\_unbiased}, \texttt{stopped\_second}\\
Theorem~\ref{thm:selection} & \texttt{DependentSelection.expected\_growing\_selection}\\
Theorem~\ref{thm:native} & \texttt{NativeFamily.own\_coupling\_counterfactual}\\
Corollary~\ref{cor:native} & \texttt{NativeNoRegret.budget\_tends\_zero} and \texttt{NativeFamily.own\_coupling\_no\_regret}\\
Theorem~\ref{thm:stability} & \texttt{StateFeedbackStability.actual\_loss}, \texttt{adaptive\_resource\_loss}, and module \texttt{FeedbackLocalOnly}, declaration \texttt{uniform\_local\_bound}\\
Theorem~\ref{thm:repair} & \texttt{NativeRepairEndpoint.counterfactual}, \texttt{no\_regret}, and \texttt{FeedbackLocalOnly.local\_only\_rate}\\
Theorem~\ref{thm:local} & \texttt{LocalGenerationCertificate.Episode.gap\_le}, \texttt{uniform\_budget}, \texttt{CompletedNativeV19.Instance.bound}, and \texttt{CompletedNativeV19.variable\_family}\\
Corollary~\ref{cor:instance} & Module \texttt{ConcreteReflection}, declarations \texttt{execution\_identity}, \texttt{conditional\_main}, \texttt{instantiated\_bound}, \texttt{family\_bound}, \texttt{square\_rate}, and \texttt{no\_regret}\\
Corollary~\ref{cor:joint-budget} & Module \texttt{JointGenerationBudget}, declarations \texttt{native\_joint}, \texttt{hard\_prefix}, \texttt{variable\_family}, and \texttt{polynomial\_allocation}\\
Theorem~\ref{thm:benchmark} & \texttt{SparseFeedbackBenchmark.paid\_sparse\_capped} and \texttt{PrimitiveJointBounds.paid\_common\_model}\\
Corollary~\ref{cor:benchmark} & \texttt{SparseFeedbackBenchmark.sparse\_no\_regret}\\
Proposition~\ref{prop:imperfect-growth} & Module \texttt{ImperfectGrowth}, declarations \texttt{pointer\_error}, \texttt{actual\_execution}, \texttt{actual\_regret}, \texttt{full\_transfer}, \texttt{actual\_no\_regret}, and \texttt{sublinear\_resources}\\
\bottomrule
\end{tabular}
\end{table}

\subsection{Meaning of end-to-end verification}
The import root includes the correspondence, clock, validated-instance, observation-law and concrete-reflection modules. The publication wrapper scores completed tasks with the proved padding identity and requires integrable intermediate endpoints. Appendix~\ref{app:v19-connections} maps its supporting derivations. Exact declaration counts, source hashes and build provenance are recorded in the release manifest. These support the thirteen numbered paper results. Declaration count is not a measure of scientific novelty.

There are three distinct levels. At the declaration level, every numbered result maps to checked Lean declarations under their displayed hypotheses. At the conditional composition level, the combined endpoint invokes proved selection, generation, and native kernel-transfer results before taking the counterfactual infimum, and the family theorem invokes this endpoint before proving convergence. This is the mathematical chain certified by the artifact. At the concrete deployment level, one would additionally need a particular reflector, its scoring and logging implementation, and proofs that they instantiate every input contract. This third claim is not made.

In particular, the \texttt{Experiment} structure contains integrability, supported-path, moment, and capacity obligations. The local observation theorem proves moment identities for its specified finite interface. An instantiation must connect that interface to the same recurrence and probability law. Likewise, \texttt{Remainder} retains acquisition, alignment, history-response, representation, and capacity conditions. The new local theorem derives generation and deployment charges, not these remaining conditions. Compilation of an implication does not convert its assumptions into facts about a real model.

The companion contains exact source hashes, per-declaration axiom logs, and a dependency-closed source set. Only \texttt{propext}, \texttt{Classical.choice}, and \texttt{Quot.sound} are allowed. There is no \texttt{sorryAx} or custom regret axiom. English statements and their Lean types were compared manually. The Lean kernel and the recorded dependency environment are trusted, rather than verified by this paper. Citations document prior work, not a claim that all cited papers have been formalized.

\paragraph{Earlier broader claims are not verified.}
The previous all-$q$ multiscale switching-rate claim is not a theorem of this version. Neither is a universal equivalence between raw-real native reflection and bit-feedback reflection, nor a universal validity claim for current-draw-dependent run stopping. The local restrictions in the revised statements are substantive and are recorded in the change audit. They are not represented as successful proofs of the broader earlier statements.

\paragraph{Position relative to related models.}
Tracking growing expert pools~\citep{mourtada2017}, sleeping bandits~\citep{nguyen2024}, and costly generative action sets~\citep{xu2026gas} already address important forms of changing availability. We do not claim that growth, generation, or counterfactual response alone is new. The focus here is the formal transfer between complete executions and the local conditions under which its error can be controlled. These references provide context, not unproved imported premises of the numbered results.

\section{A fully specified native reflection instance}
\label{app:concrete}
This appendix proves Corollary~\ref{cor:instance} without changing the general theorems. It provides a mathematical existence witness, not a measurement of a deployed language model. The instance is deliberately validation-dominant. A complete validator identifies a useful program, so the native controller can use it without learning an external selection rule. Generation can nevertheless fail, the registered library genuinely grows, deployment has a nonzero discrepancy, and all operation prefixes respect the same hard budget. The reference Bandit is used only in the proof. The construction does not establish joint tightness of all terms in the general bound.

\paragraph{Scope of the instance.}
The guarantee concerns the specified frozen template generator, complete validator, retention rule, native schedule, and localized deployment check. The no-regret conclusion is along a horizon-parameterized family, rather than a single anytime controller. Applying this design pattern to a language-model agent requires establishing the corresponding generation, validation, retention, resource, and deployment conditions for that agent. The example establishes joint feasibility of the native-transfer interfaces and a concrete benefit of reflection within this model.

\subsection{Tasks, programs, and the frozen reflection mechanism}
Fix integers $k,N\ge1$ and $T=k+N$. Task $t$ presents $x_t\in\{0,1\}$ and requests the program output $x_t$. The context sequence can be fixed arbitrarily before execution. A structured result records the answer, program identity, validation status, and remaining resources. The common task scorer is zero for the correct answer and one otherwise. It is applied identically to the learner and every comparator.

There are three immutable program identities. Their executable content is
\begin{equation}
 P_0(x)=x+1,\qquad P_1(x)=x+2,\qquad P_2(x)=x.
 \label{eq:instance-programs}
\end{equation}
Initially only $P_0,P_1$ are registered. A complete validator tests both inputs, so it accepts $P_2$ and rejects the other two for this specification. This is possible because the template family and task domain are finite. We do not assume that finite tests establish correctness for arbitrary programs.

The fixed reflection kernel supports two prompts. A generic repair prompt returns $P_2$ or $P_1$ with equal probability. A guided template prompt returns $P_2$ deterministically. These are two specified input modes of the same unchanged kernel, available to both learner and comparator. The native learner uses the generic prompt. The comparator below uses the guided prompt. No learned weight, success probability, or hidden optimal-card label is supplied to the native controller.

After each of the first $k$ completed tasks, the native learner makes one generic generation and validation attempt. Every successful candidate is registered immediately unless its immutable identity is already present. The controller retains the success flag and uses $P_0$ throughout this prescribed preparation phase. Thus early registration does not force immediate execution. After the $k$th attempt, it executes $P_2$ if present and $P_0$ otherwise. This batching choice incurs an explicit $k$-task acquisition cost. It is a behavior of this particular controller, not a restriction on the rapid-registration interface.

Ordinary execution and execution with an additional check are legal runtime operations for every procedure. On the first subsequent task only, the native schedule invokes the check, which drops a successful output with probability $1/4$, replacing it by the fallback answer. Its private bit is not retained in memory or the structured result. It does not modify registered programs, future tasks, or the future control rule. Later tasks execute the retained program normally. The clean reference performs the same charged check but leaves the output unchanged, so both resource counters agree. The native controller is consequently neither the reference Bandit nor identical to its own clean deployment.

\subsection{Own-history comparators and prefix resources}
The comparator class $\mathcal M_0$ contains two procedures. Procedure $\mu_0$ always uses $P_0$. Procedure $\mu_1$ makes one guided generation, validates and registers its output before the first task, and then uses $P_2$. Pre-task internal operations are legal for both procedures and the learner. The native learner simply follows its prescribed later schedule. The environment and scoring rule are unchanged. Consequently
\begin{equation}
 J_T(\mu_0)=T,\qquad J_T(\mu_1)=0,
 \qquad \inf_{\mu\in\mathcal M_0}J_T(\mu)=0.
 \label{eq:instance-comparators}
\end{equation}
These are separate counterfactual executions, not comparator actions evaluated against the native memory. Nonnegative scores also imply that adding other legal comparators cannot lower this infimum below zero. The two-element class is sufficient for the formal witness.

Each attempt is charged two units at initiation, one for generation and one for its validation. Registration of the new identity consumes one further unit, and the deployment check consumes one unit. Initial registrations are part of the common initial condition. If $r$ attempts have been initiated and $u,v\in\{0,1\}$ indicate whether registration and the check have been charged, every native prefix satisfies
\begin{equation}
 C_r=2\min\{r,k\}+u+v\le 2k+2.
 \label{eq:instance-hard}
\end{equation}
The guided comparator costs three units, at most $2k+2$. The fallback comparator costs zero. Upfront charging also covers prefixes inside an unfinished attempt. Failed attempts are not refunded. Resource charges restrict admissibility and are not added again to the task scorer. The separate local deployment resource is one check, so its activity budget is $B_A=1$, while $B_C=0$.

\subsection{The native probability law and actual task losses}
Let $H_k=(\zeta_1,\ldots,\zeta_k)\in\{0,1\}^k$ be the generic trial outcomes, with success encoded by one. For a prefix $h_j$, the conditional next-outcome probability is $1/2$ for either value. The resulting finite law is the recursively normalized product law, not an assumed expected-loss formula. Define
\begin{equation}
 \xi(h)=\mathbf1\{\zeta_1=\cdots=\zeta_k=0\},
 \qquad \mathbb P(\xi(H_k)=1)=2^{-k}.
 \label{eq:instance-miss}
\end{equation}
The deployment bit $D$ has probabilities $\mathbb P(D=1)=1/4$ and $\mathbb P(D=0)=3/4$, independently of the trial history. Pre-sampling these bits is only a representation of the law. Decisions access each trial result after its generation and the deployment bit only at its check.

For each realized $(h,d)$, the completed-task loss is
\begin{equation}
 Y_t(h,d)=
 \begin{cases}
 1,&0\le t<k,\\
 \xi(h)+\mathbf1\{t=k,\ d=1\}(1-\xi(h)),&k\le t<T.
 \end{cases}
 \label{eq:instance-task}
\end{equation}
All scores lie in $[0,1]$. Summing this actual process before taking expectations gives
\begin{equation}
 \sum_{t<T}Y_t(h,d)=k+N\xi(h)+d(1-\xi(h)),\qquad
 J_T(A)=\mathbb E_{H_k}\left[k+N\xi(H_k)+\tfrac14(1-\xi(H_k))\right].
 \label{eq:instance-execution}
\end{equation}
The Lean declarations \texttt{task\_sum}, \texttt{conditional\_execution}, and \texttt{execution\_identity} establish this correspondence. Thus the expected loss is derived from task execution rather than introduced as an unconstrained endpoint.

\subsection{Local generation and full-state deployment conditions}
Each unsuccessful trial history has conditional success probability $\rho=1/2$. A successful card is correct on every subsequent task and is retained. Therefore $\delta=0$. The available post-preparation path uses $P_2$ when generated and a fallback otherwise. Its difference from the ideal comparator path is $N\xi(H_k)$. Applying the existing local-generation theorem gives
\begin{equation}
 G=\mathbb E_{H_k}[N\xi(H_k)]\le N2^{-k}.
 \label{eq:instance-generation}
\end{equation}
Here the episode delay is zero because the entire preparation loss is already charged as $F=k$. Charging another $k$ in $G$ would duplicate the same loss. The exact failure probability follows as well by induction on the independent binary trials.

At the deployment check, the future-relevant state consists of the retained success flag, the output-failure bit, and the deterministic task and resource counters. The registered library is determined by the success flag. Raw historical records can also be retained, but neither kernel uses the past output bit in a future decision. The native and clean kernels preserve the same memory and counters. Their complete stochastic state distributions have
\begin{equation}
 \TV(K,L)=\tfrac14(1-\xi(h)).
 \label{eq:instance-tv}
\end{equation}
After that task they induce identical future control and losses. Conditional on the completed prefix $h$, the retained memory is fixed and the future loss is common to both output branches. Their continuation-value span is at most one. The local discrepancy is at most $1/4$ and the check costs one unit. Averaging the resulting conditional difference, or directly using~\eqref{eq:instance-execution}, gives the deployment certificate
\begin{equation}
 E_A\le 1\cdot\tfrac14\cdot B_A=\tfrac14,
 \qquad E_C=0.
 \label{eq:instance-deployment}
\end{equation}
The formalization computes TV on the joint memory-output state, proves that later task scores ignore the slip bit, and establishes the exact endpoint difference. The span-one statement is conditional on the fixed memory, not a global span assertion across successful and unsuccessful memories. Those memories can have very different future values. This instance therefore supplies the deployment edge of Theorem~\ref{thm:native} directly. It does not assume the stronger global range contract of Theorem~\ref{thm:stability} with $\bar H=1$ on the unconditioned mixed-memory state space.

\subsection{The reference Bandit and the native selection bridge}
Condition on a completed generation history $h$, before drawing any reference action. Consider a two-arm reference pool. Arm zero is $P_0$. Arm one is $P_2$ if available, and otherwise $P_1$. These are distinct immutable identities in either case. The unused registered fallback need not enter this comparison pool. Registration is causal because the reference experiment begins after the $k$ attempts. The native registered library grows on every successful trial trajectory, even though the subsequent comparison pool has two members.

The reference algorithm initializes both probabilities at $1/2$ and performs ordinary exponential weighting with learning rate $\eta>0$ and identity transition. Only the selected score enters the update. Its score is
\begin{equation}
 \ell_h(0)=1,\qquad \ell_h(1)=\xi(h),\qquad q_r(i)=1.
 \label{eq:instance-reference-loss}
\end{equation}
The observation kernel is deterministic conditional on the selected identity. It therefore satisfies the accepted-score identity exactly. Positivity of exponential weights gives full support at every finite time. The comparator path is arm one, with code $C=\log2$ and no switches. The standard conditional second-moment calculation gives $V_{\rm est}\le2N$. The constructed Lean \texttt{Program} recursively generates its draw probabilities and observations. Its capacity and code conditions are proved, not supplied as assumptions.

Let $J_{\rm sel}(h)$ be its expected post-preparation loss and $J_{\rm av}(h)=N\xi(h)$. Both arms have loss at least $\xi(h)$, so
\begin{equation}
 N\xi(h)\le J_{\rm sel}(h),\qquad
 J_{\rm sel}(h)-N\xi(h)\le\frac{\log2}{\eta}+\eta N.
 \label{eq:instance-selection}
\end{equation}
The first inequality is the crucial native bridge. The clean native controller has total loss $k+N\xi(h)$, hence its acquisition edge is at most $k$. This is a one-sided performance comparison proved from validation, not an assertion that native sampling probabilities equal the Bandit probabilities. On a successful history, the native clean controller always uses the correct card. The initial reference draw uses the fallback with probability $1/2$, and the deployed native controller uses it with probability $1/4$ on its first post-preparation task. The controllers are genuinely distinct.

\subsection{A causal coupling with the correct native marginal}
Write $w_k(h)$ for the recursively generated trial probability, $s(d)$ for the deployment-bit probability and $Q_{\eta,h}(g)$ for the law of a complete $N$-slot reference trajectory. The latter law comes from the constructed Bandit recursion, not from native action logs. On the finite product space define
\begin{equation}
 \Gamma(h,d,g)=w_k(h)s(d)Q_{\eta,h}(g),
 \qquad \sum_{h,d,g}\Gamma(h,d,g)=1.
 \label{eq:instance-joint}
\end{equation}
Each factor is nonnegative and normalized at its proper stage. The reference uses $h$ only after all $k$ trials have completed. It does not use the subsequent deployment bit. Thus this conditioning does not reveal future generation outcomes to an earlier draw.

The native output functional in~\eqref{eq:instance-task} does not depend on $g$. Marginalizing the freshly generated reference path gives
\begin{equation}
 \sum_{h,d,g}\Gamma(h,d,g)\sum_{t<T}Y_t(h,d)
 =\sum_{h,d}w_k(h)s(d)\sum_{t<T}Y_t(h,d)=J_T(A).
 \label{eq:instance-marginal}
\end{equation}
Both comparators have deterministic scores on their own executions, so they can be coupled to this space without changing either native law. The learner marginal is the same for both comparator choices. The Lean declarations \texttt{joint\_law\_normalized} and \texttt{native\_marginal} prove these identities using normalization of the actual reference recursion. This supplies the compatibility missing from a comparison of isolated action marginals.

\subsection{All certificates in one experiment}
For the best comparator $\mu_1$, set all ideal comparator-side endpoints to zero. This is its actual loss, represented exactly by $P_2$, whether or not that program is available to the learner. The available-to-ideal difference is precisely the generation edge. The fallback comparator has nonnegative loss, so the same final upper bound applies to it. No learner marginal depends on this comparator choice.

Use one task per reference slot, no response contamination and no path switches. In the notation of Theorem~\ref{thm:native}, the proved common budgets are
\begin{equation}
 \begin{gathered}
 b=1,\quad S=\hresp=N_b=0,\quad C=\log2,\quad V=2N,\\
 F=k,\quad G=N2^{-k},\quad E_A=\tfrac14,\quad
 E_C=D_A=D_C=E_{\rm nf}=0.
 \end{gathered}
 \label{eq:instance-ledger}
\end{equation}
Generation cost constrains~\eqref{eq:instance-hard}. The preparation task loss appears only in $F$. Failed post-preparation opportunities appear only in $G$. The deployment check appears only in $E_A$. Exact validation and the absence of an endogenous environment state make the remaining edges zero by construction, not by a generic assumption about reflective agents.

Apply the unchanged native-transfer theorem conditionally on $h$, using the actual reference recursion. Then average over the normalized generation law and use the local-generation theorem. Equation~\eqref{eq:instance-comparators} identifies the own-history comparator infimum. This proves
\begin{equation}
 R_T(A,\mathcal M_0)\le k+N2^{-k}+\tfrac14+\frac{\log2}{\eta}+\eta N.
 \label{eq:instance-general-bound}
\end{equation}
In Lean, \texttt{conditional\_main} invokes \texttt{ModelMain.finite\_controller\_transfer}. The declaration \texttt{instantiated\_bound} averages that theorem and invokes \texttt{generation\_bound}. The proof therefore uses the existing selection and transfer chain rather than merely proving a numerically similar inequality in isolation.

\subsection{The horizon-indexed family and its scope}
For $k=n+1$, put $N=k^2$, $T=k^2+k$ and $\eta=1/k$. Equation~\eqref{eq:instance-general-bound} becomes the first inequality in Corollary~\ref{cor:instance}. Induction gives $k\le2^k$, hence $k^2 2^{-k}\le k$. Since $k^2\le T$, this is an $O(\sqrt T)$ bound along the specified horizons. More explicitly,
\begin{equation}
 \frac{R_T(A,\mathcal M_0)}T
 \le \frac{2+\log2}{k}+2^{-k}+\frac1{4k^2}
 \longrightarrow0.
 \label{eq:instance-rate}
\end{equation}
The positive-part limit is checked in Lean. Without reflection, $\mu_0$ has regret exactly $T$, so the construction is not a comparison against the learner itself. It is an existence result for a restricted verified-template model, not an optimal-rate theorem, an arbitrary-language-model guarantee, or a single horizon-free controller. The constant coverage and complete validator are strong, explicit structural features. Removing them requires additional analysis. The model nevertheless demonstrates that genuine rule generation, retained native control, nonzero deployment error and hard resource feasibility can coexist with the unchanged theorem and a nonvacuous sublinear conclusion.

\section{Joint generation quality and hard resources}
\label{app:joint-budget}
This appendix proves Corollary~\ref{cor:joint-budget}. The existing transfer theorem is unchanged. The new step couples the required policy quality to conditional coverage, trial quotas, and actual hard resources. A fixed design determines these parameters and a causal trial schedule. Selecting a design for a horizon does not replace the native controller by a Bandit controller. The guarantee applies when that native procedure satisfies the stated interfaces for the selected design.

\subsection{Contracts and the meaning of the parameters}
Fix a comparator before execution. For episode $e$, quality tolerance $\delta_e$ bounds the conditional expected available-minus-aligned task gap after success. A success means that a comparator-suitable rule has been generated, validated, retained, and made causally usable. It is an analysis predicate, not an optimal-policy label supplied to the native procedure. The lower bound $\rho_e(\delta_e)\in(0,1]$ must hold after every still-unsuccessful trial history. Improving the requested quality may decrease this bound. A separate statement about text quality alone would not imply the task-gap contract.

The quota $k_e\in\mathbb N$ bounds attempted generation and validation operations. Successful episodes may stop early, with inactive trials padded for analysis. Such padding consumes no additional resource and cannot undo a success. The worst-case per-attempt envelope $c_e(\delta_e)$ includes failures. Costs outside this envelope, such as registration, are included in $B_T^{\rm other}$. The plan is fixed before play, though its trial kernels and legal decisions may depend on completed history. Random redesign of the quotas requires a separate predictable-envelope argument.

Episode slots refer to one learner schedule shared by all comparator analyses. Success predicates may depend on the comparator, but the same physical attempt is charged once. Resource costs are not summed over counterfactual comparators. The common plan supplies uniform envelopes for their separate generation-gap proofs.

For a legal learner prefix $h$, let $u_e(h)$ be its initiated attempt count and $C^{\rm other}(h)$ its remaining charged costs. Assume each attempt reserves its cost before unfinished internal work can violate a budget, all costs are nonnegative, $u_e(h)\le k_e$, and $C^{\rm other}(h)\le B_T^{\rm other}$. Then
\begin{equation}
 C_A(h)\le C^{\rm other}(h)+\sum_eu_e(h)c_e(\delta_e)
 \le B_T^{\rm other}+\sum_ek_ec_e(\delta_e)\le B_T.
 \label{eq:joint-prefix}
\end{equation}
This proves every-prefix hard feasibility, including prefixes inside a failed or unfinished attempt. It does not replace the comparator admissibility definition. Each comparator still executes on its own history and satisfies its own original hard constraint. In particular, a comparator cannot borrow a card or unused resource from the learner. These resource costs constrain execution rather than adding a second loss penalty.

\subsection{Conditional failure, task gaps, and native regret}
Let $F_{e,j}$ be the event that the first $j$ trials in episode $e$ have not succeeded, and let $\mathcal F_{e,j}$ record those completed trials. Under the actual episode law $\mathbb P_e$, conditional coverage implies
\begin{equation}
 \mathbb P_e(F_{e,j+1})
 =\mathbb E_{\mathbb P_e}\!\left[
 \mathbf1_{F_{e,j}}\mathbb P_e(F_{e,j+1}\mid\mathcal F_{e,j})\right]
 \le [1-\rho_e(\delta_e)]\mathbb P_e(F_{e,j}).
 \label{eq:joint-hazard}
\end{equation}
The first equality uses that $F_{e,j+1}\subseteq F_{e,j}$ and the tower rule. On a successful history the indicator is zero. On an unsuccessful history the conditional failure probability is at most $1-\rho_e(\delta_e)$. Starting from probability one and applying this argument $k_e$ times gives
\begin{equation}
 \mathbb P_e(F_{e,k_e})\le[1-\rho_e(\delta_e)]^{k_e}
 \le\exp\{-k_e\rho_e(\delta_e)\}.
 \label{eq:joint-failure}
\end{equation}
For the second inequality, $1-x\le e^{-x}$ for $x\in[0,1]$, both sides are nonnegative, and raising to the nonnegative integer power preserves order. No independence assumption was introduced.

On every completed generation history, the conditional expected per-task gap is bounded by one during the $\tau_e$ delay tasks and by $\delta_e+\mathbf1_{F_{e,k_e}}$ for the $n_e$ subsequent tasks. Integrating under the same episode law gives
\begin{equation}
 \mathbb E_{\mathbb P_e}\left[
   \sum_{t<\tau_e}g^{\rm pre}_{e,t}
   +\sum_{t<n_e}g^{\rm post}_{e,t}\right]
 \le\tau_e+n_e\left[\delta_e+e^{-k_e\rho_e(\delta_e)}\right].
 \label{eq:joint-episode}
\end{equation}
The random variables $g^{\rm pre}_{e,t}$ and $g^{\rm post}_{e,t}$ are conditional task-gap functions on completed trial histories, as in Appendix~\ref{app:localgeneration}. Their sums must represent the actual available-to-aligned endpoint difference. Summing these identities and bounds establishes the generation certificate. Uniformity over the original comparator class permits the counterfactual infimum in Theorem~\ref{thm:native}. Its remaining certificates, reference selection theorem, and own-law realizations are unchanged. Substitution proves~\eqref{eq:joint-native}.

The quantity $\overline U_T$ contains all remaining reference and deployment terms. In particular, direct deployment certificates and those proved by Theorem~\ref{thm:stability} are both admissible. Replacing a direct certificate with a local proof does not create an additional copy of the same endpoint difference. Likewise, a preparation delay already charged in $F$ has zero additional delay in~\eqref{eq:joint-episode}.

In Theorem~\ref{thm:stability}, local shift $\varepsilon_t$ is weighted by downstream sensitivity $h_t$, along the actual law $\mathbb P_K$. Resource use limits the accumulated charge, not the union of all possible reflection times. Bounded card lifetime alone does not bound $h_t$, since descendants and environment changes can retain the effect.

For the asymptotic claim, the uniform envelopes in the corollary give
\begin{equation}
 \frac{R_T}{T}\le
 \frac{\overline U_T+\sum_e\tau_e}{T}
 +\delta_T+e^{-\lambda_T}.
 \label{eq:joint-normalized}
\end{equation}
The right side tends to zero. Continuity of the positive-part map and the nonnegative lower bound on $\max\{0,R_T/T\}$ yield the claimed limit. This argument derives a regret bound from native execution witnesses before taking a limit. It does not assume that regret itself is sublinear. A vanishing coverage probability is allowed, provided cumulative exposure $k_e\rho_e(\delta_e)$ still diverges uniformly. Increasing $k_e$ alone does not establish this property.

\subsection{Explicit quotas and a polynomial quality-cost family}
For a desired log-failure level $\Lambda>0$, a coverage lower bound $\rho_\star>0$, and per-attempt cost envelope $c_\star\ge0$, define
\begin{equation}
 k=\left\lceil\frac{\Lambda}{\rho_\star}\right\rceil,
 \qquad k\rho_\star\ge \Lambda,
 \qquad kc_\star\le\left(\frac{\Lambda}{\rho_\star}+1\right)c_\star.
 \label{eq:joint-quota}
\end{equation}
The first bound follows from $\lceil x\rceil\ge x$. The second follows from $\lceil x\rceil<x+1$ for $x\ge0$, weakening the strict inequality when $c_\star=0$. Equations~\eqref{eq:joint-failure} and~\eqref{eq:joint-quota} bound the failure probability by $e^{-\Lambda}$. A design can therefore check hard feasibility using these explicit quotas, including the integer-rounding cost.

One illustrative class has constants $a_{\rm g}\in(0,1]$, $c_0>0$, integer exponents $d_{\rm g},d_{\rm v}\ge0$ with $\nu=d_{\rm g}+d_{\rm v}>0$, and, uniformly over episodes and comparators,
\begin{equation}
 \rho_e(\delta)\ge a_{\rm g}\delta^{d_{\rm g}},
 \qquad c_e(\delta)\le c_0\delta^{-d_{\rm v}},
 \qquad 0<\delta\le1.
 \label{eq:joint-polynomial}
\end{equation}
The exponent $d_{\rm g}$ describes how quickly suitable candidates become rare as quality improves. The exponent $d_{\rm v}$ describes increasing generation or validation cost. These are additional properties of this illustrative class, not new assumptions imposed on every model in the main theorem. Integer exponents suffice for this formally checked example.

Choose $\rho_\star=a_{\rm g}\delta^{d_{\rm g}}$ and the quota in~\eqref{eq:joint-quota}. Since $\delta^{d_{\rm g}}\le1$, the cost of all attempts in any episode satisfies
\begin{equation}
 kc_e(\delta)
 \le\left(\frac{\Lambda}{a_{\rm g}\delta^{d_{\rm g}}}+1\right)
       \frac{c_0}{\delta^{d_{\rm v}}}
 =\frac{c_0(\Lambda/a_{\rm g}+\delta^{d_{\rm g}})}{\delta^\nu}
 \le\frac{c_0(\Lambda/a_{\rm g}+1)}{\delta^\nu}.
 \label{eq:joint-polynomial-cost}
\end{equation}
Let $E_T\ge1$ be the deterministic number of episode slots, including padding, and let $b_T^{\rm gen}>0$ be the allocation for these attempts after reserving other costs. This is a resource allocation, distinct from the run cap $b$. Define
\begin{equation}
 z_T=\frac{c_0E_T(\Lambda/a_{\rm g}+1)}{b_T^{\rm gen}},
 \qquad 0<z_T\le1,
 \qquad\delta_T=z_T^{1/\nu}.
 \label{eq:joint-root}
\end{equation}
The positive real root satisfies $0<\delta_T\le1$ and $\delta_T^\nu=z_T$. Substituting into~\eqref{eq:joint-polynomial-cost} and summing the $E_T$ slots bounds total attempt cost by $b_T^{\rm gen}$. Hence $b_T^{\rm gen}+B_T^{\rm other}\le B_T$ ensures the original hard constraint. This is a sufficient, explicit design, not an infimum over incompatible parameter choices.

For $T\ge2$, take $\Lambda=2\log T$. Then $e^{-\Lambda}=T^{-2}$, and the aggregate failure contribution is at most $T^{-1}$ when $\sum_en_e\le T$. Thus~\eqref{eq:joint-normalized} specializes to
\begin{equation}
 \frac{R_T}{T}\le
 \frac{\overline U_T+\sum_e\tau_e}{T}
 +\left[\frac{c_0E_T(2\log T/a_{\rm g}+1)}{b_T^{\rm gen}}\right]^{1/\nu}
 +T^{-2}.
 \label{eq:joint-polynomial-rate}
\end{equation}
If the bracketed ratio tends to zero, the quality term tends to zero by continuity of a positive power at zero. Remaining terms must still vanish. For fixed $a_{\rm g},c_0$, the familiar sufficient scaling $E_T\log T=o(b_T^{\rm gen})$ makes that ratio vanish. When additional delay and deployment effects are bounded by a fixed constant times $b_T^{\rm gen}$, one also chooses $b_T^{\rm gen}=o(T)$ and verifies the unaffected reference terms. Resource availability itself need not be sublinear, since an agent need not spend its entire allowance.

The values of coverage and cost envelopes determine this design. An algorithm adapting to unknown envelopes is a further learning problem, not a consequence of taking an infimum over designs. Any such optimization must retain the dependence of delay, observation, switching, and deployment terms on the same design.

\subsection{Compatibility with the native verified-template instance}
Corollary~\ref{cor:instance} and Corollary~\ref{cor:joint-budget} use the same transfer theorem with different choices of the local generation parameters. In the verified-template instance, there is one episode with $k_e=k$, $n_e=N$, $\rho_e(0)=1/2$, and $\delta_e=0$. A trial costs two units, and registration plus the deployment check reserve two further units. The hard bound is exactly $2k+2$. Preparation is already charged as $F=k$, so $\tau_e=0$ in the generation ledger. Its direct deployment charge remains $1/4$.

The generic exponential failure envelope becomes $e^{-k/2}$. The exact finite-model probability is $2^{-k}$ and is sharper. Keeping that exact bound preserves Corollary~\ref{cor:instance} verbatim. Its native controller and counterfactual comparators are unchanged. In particular, the new result does not demand the stronger global continuation-range contract discussed in Appendix~\ref{app:concrete}.

The Lean declaration \texttt{concrete\_compatibility} checks both conclusions on the original model. It retains the actual prefix-cost function and derives the exponential-envelope regret bound from the already proved native regret. Thus compatibility includes the concrete execution and hard resources, rather than only a comparison of abstract rate expressions.

The polynomial specialization concerns $\delta>0$ and illustrates a genuine precision-cost tradeoff. The existing exact-repair point has $\delta=0$ and uses the more general corollary directly, rather than substituting zero into a negative-power cost formula. This distinction is the reason the two results are compatible. One establishes a nonempty concrete native model, while the other explains how a family of legal designs can allocate resources as quality requirements change.

\paragraph{Formal proof chain.}
Module \texttt{JointGenerationBudget} reuses the normalized finite trial law and proved failure recurrence. It constructs the missing generation certificate from episode kernels and exact endpoint identities, then calls \texttt{NativeFamily.own\_coupling\_counterfactual}. Direct or locally established deployment certificates remain explicit inputs. Its variable-domain family theorem invokes this native endpoint before proving the limit. Separate declarations verify every-prefix costs, integer quotas, polynomial costs, logarithmic failure targets, and positive-root choices. The scalar power manipulations are supporting steps rather than substitutes for the native semantic connection.

\paragraph{Interpretation of local generation reliability.}
Local generation reliability gives $\mathbb P_e(\xi_e=1)\le(1-\rho_e)^{k_e}$ by successive conditional failure bounds. Integrating task gaps and summing episodes proves~\eqref{eq:generatedbudget}. More trials reduce failure, better retained policies reduce quality error, and faster acquisition reduces delay. Theorem~\ref{thm:local} combines these effects with both deployment bounds and the remaining contracts. The Lean family endpoint applies this completed-task bound before taking the no-regret limit. Appendix~\ref{app:localgeneration} supplies the proof.

All trials consume resources, including failures. Delay is charged here or in $F$, never both. Local generation reliability must hold uniformly over the comparator class. Increasingly precise, disjoint targets need not retain constant coverage. The remaining contracts must hold jointly.

\section{Notation, types, and correspondence to the formal model}
\label{app:notation}
The notation follows the usual distinction between losses $\ell$, loss estimates $\widehat\ell$, unnormalized weights $w$, and sampling probabilities $p$ in online learning~\citep{lattimore2020,mourtada2017}. A bold $\mathbf v$ is a path, never its loss. The functions $J_T$ and $J$ with semantic superscripts denote expected totals, whereas $L$ denotes a realized total. The separately defined $J_r$ is a sampled reference identity, not a loss. Subscripts and superscripts distinguish the native model, reference slots, local stability model, and binary benchmark. Dummy variables in explicitly local proofs are not additional global model parameters.

\subsection{Histories, probability laws, and conditioning}
An operation history $h\in\mathcal H_s^{\rm op}$ records $s$ completed internal operations. A task history $H_t\in\mathcal H_t^{\rm task}$ contains the initial record and $t$ completed task records. These are not the same prefix space. A task-record kernel packages the internal computation between task boundaries, assuming it terminates and induces a measurable kernel. The theory does not prove termination of arbitrary language-model programs. The symbol $\mathsf E$ in Definition~\ref{def:history} is an admissibility relation, whereas $\E$ is expectation.

The runtime is the mechanism that interprets a selected card, invokes tools or primitive actions, and updates the recorded state. A card kernel $\pi_i(da\mid x,m)$ assigns a probability measure to actions $a$, conditional on context $x$ and memory $m$. Here $da$ is integration notation, not a product involving an undefined parameter $d$. The same convention applies to the operation kernel $\mu(du\mid h)$. For a measurable function $f$, execution averages $f(a)$ as $\int f(a)\pi_i(da\mid x,m)$.

For a fixed comparator, $\Gamma_\mu$ is a joint proof law on a space $\Omega_\mu$. A realization is a measurable map from that space to a complete trajectory. The subscript $\#$ is the measure-pushforward operator, not a cardinality symbol. For a measurable trajectory event $D$, $((r_A)_\#\Gamma_\mu)(D)=\Gamma_\mu(r_A^{-1}(D))$, where $r_A^{-1}(D)=\{\omega\in\Omega_\mu\mid r_A(\omega)\in D\}$. Thus the probability of $D$ is obtained by measuring the source outcomes mapped into it. The required equalities $(r_A)_\#\Gamma_\mu=\mathbb P_A$ and $(r_\mu)_\#\Gamma_\mu=\mathbb P_\mu$ preserve the complete native trajectory distributions, not merely one-step action marginals. The learner law $\mathbb P_A$ cannot change with the comparator used to prove the bound.

The notation $\E_{J_r,O_r\mid h_r}$ means a finite conditional expectation under masses $p_r(i)k_r(i,o)$ after fixing the completed pre-draw history. Equivalently, in a compatible reference process it is conditional expectation given $\mathcal F_r^-$. The sigma-field is generated by the completed prefix and contains no current draw or future observation. All slot-dependent objects may be functions of that prefix. The notation suppresses these arguments but does not assert determinism. By contrast, $\E_{\Gamma_\mu}$ integrates the entire coupled experiment, including its random prefixes. The law $\mathbb P_{\rm stop}^{(r)}$ is a prefix marginal of the recursively generated stopped law. It is not conditioning on the terminal stopping time.

The endpoint $\lref_r(i)$ is fixed before the current draw. Operationally it is the expected task-loss total obtained from the specified reference execution at that prefix and identity. The score $y_r(i,o)$ is an observable proxy. Equation~\eqref{eq:accepted} is a calibration assumption relating that proxy to the independently specified target. Defining the target to be an accepted-score average after the fact would not establish correspondence to task loss. Neither a positive $q_r(i)$ nor a positive sampling probability $p_r(i)$ supplies that correspondence.

\begin{longtable}{@{}p{0.27\linewidth}p{0.68\linewidth}@{}}
\caption{Semantic and algorithmic notation. Domains and probability laws are part of the definitions.}\label{tab:notation-core}\\
\toprule Symbol & Meaning and scope\\\midrule
\endfirsthead
\toprule Symbol & Meaning and scope, continued\\\midrule
\endhead
\bottomrule\endfoot
$T,t,s,r,R,N$ & Task horizon and task index, internal-operation index, reference-slot index, last padded slot, and slot cap $N=R+1$. Defined in Sections~\ref{sec:model} and~\ref{sec:stopping}.\\
$\mathsf U,\mathsf O,\mathsf U(h),\mathsf E$ & Operation and result spaces, legal next operations, and admissible extension relation. These describe operational legality, not resource penalties.\\
$\mathsf X,\mathsf M,\mathsf A$ & Context, memory, and primitive-action spaces. Lowercase $x,m,a$ are their elements. The benchmark phase count $m$ is local to Section~\ref{sec:benchmark}.\\
$\mathcal H_s^{\rm op},\mathcal H_t^{\rm task}$ & Operation-history and task-history spaces. Random task prefix $H_t=(Z_0,\ldots,Z_t)$, with $Z_t\in\mathcal Z_t$. Lowercase $h$ is a realized history.\\
$\pi_i,\mu,A,\mu_A,\mathcal M$ & Execution-card action kernel, comparator meta-policy, complete native learner, its operation-selection kernel, and nonempty comparator class. $i$ is immutable executable identity, not a primitive action.\\
$\sigma_t,Y_t,L_T,J_T,R_T$ & Common task scorer, scalar task loss, realized cumulative loss, expected cumulative loss, and own-history regret from~\eqref{eq:objective}. Unqualified $R_T$ abbreviates $R_T(A,\mathcal M)$.\\
$\mathbb P_A,\mathbb P_\mu,\Gamma_\mu$ & Own-trajectory laws and comparator-specific coupling with the prescribed native marginals. $\Omega_\mu,\omega$ denote its sample space and outcome.\\
$r_A,r_\mu,(r_A)_\#\Gamma_\mu$ & Measurable trajectory realizations and the induced trajectory law. The subscript $\#$ denotes measure pushforward. The letter $r$ alone indexes a reference slot.\\
$\cost_t^\mu,B_t$ & Cumulative resource use and deterministic hard budget at task boundary $t$. Non-task operationwise restrictions must be imposed separately.\\
$\mathcal F_r^-,h_r,\mathcal I_r$ & Pre-draw sigma-field, realized completed prefix, and nonempty finite available identity set. The set and local kernels may be history-dependent.\\
$\mathbf1\{r\text{ active}\}$ & Indicator that the reference process has not terminated and the completed history permits slot $r$. It is $\mathcal F_r^-$-measurable, not a card-availability or feedback-acceptance indicator.\\
$\nu_r,w_r,W_r$ & Probability prior, nonnegative unnormalized weight vector, and total weight $W_r=\sum_iw_r(i)>0$. $W_r$ is not a task record $Z_t$.\\
$p_r,J_r,\Theta_r,\alpha$ & Sampling probability vector, sampled reference identity, row-stochastic sharing transition, and share rate in $(0,1)$. The general transition need not use the Fixed-Share specialization.\\
$k_r,O_r,M_r,q_r$ & Conditional outcome kernel, random outcome, binary acceptance mask, and acceptance probability conditional on the selected identity. $q_r$ is not the probability of selecting that identity.\\
$\lambda_r,y_r,d_r,\lref_r$ & Reference task-loss total, observed proxy score, its predictable upper bound, and reference expected loss per identity. Equation~\eqref{eq:accepted} links the two scores.\\
$\widehat\ell_r,e_r$ & Importance-weighted estimated loss and zero-conditional-mean learner--comparator estimation noise. Squaring a vector means coordinatewise squaring.\\
$\mathbf v^\mu,v_r,C(\mathbf v),C$ & Predictable available comparator path, its identity coordinate, realized negative log path mass, and deterministic uniform code bound. Dependence on $\mu$ is suppressed locally.\\
$V_{\rm est},V$ & Raw cumulative second-moment budget in~\eqref{eq:capacity} and its scaled envelope with $V_{\rm est}\le bV$. Neither is a centered variance.\\
$\mathbb P_{\rm stop},\varpi_r,F_r,f_r$ & Stopped law, completed-prefix probability mass, cumulative event score, and its one-step increment in Appendix~\ref{app:stopped}. $F_r$ there is not the scalar foreground budget $F$.\\
$\omega_r^{\mathbf v}$ & Unnormalized contribution of a single supported path to weight $w_r(v_r)$, as defined in~\eqref{eq:v14-path-weight}.\\
\end{longtable}

\subsection{Endpoint types and the single accounting chain}
Fix $T$ and $\mu$. Each intermediate endpoint is an explicitly chosen integrable score random variable on the coupling space. Its expectation is the displayed scalar $J$. The superscripts describe its role, not an algebraic operation that could manufacture the endpoint automatically. An application must supply the random variables and their semantic identifications before proving their inequalities. In particular, $\Jav$ is an expected score, whereas $\mathbf v$ is a random predictable path.

\begin{longtable}{@{}p{0.27\linewidth}p{0.68\linewidth}@{}}
\caption{Expected endpoints, directed charges, and exact Lean field names. All expectations use $\Gamma_\mu$.}\label{tab:notation-endpoints}\\
\toprule Endpoint or charge & Definition, outgoing comparison, and correspondence\\\midrule
\endfirsthead
\toprule Endpoint or charge & Definition, outgoing comparison, and correspondence, continued\\\midrule
\endhead
\bottomrule\endfoot
$\Jna$ & Expected native learner loss $J_T(A)$. Expectation of \texttt{realA}. Its difference from $\Jra$ is bounded by $E_A$.\\
$\Jra$ & Expected reference learner endpoint, expectation of \texttt{refA}. Its difference from $\Jsel$ is bounded by $F+\mathcal P_b+D_A$.\\
$\Jsel$ & $\sum_{r=0}^R\E_{\Gamma_\mu}\lref_r(J_r)$, identified with \texttt{Experiment.selectedTotal}. Its comparison to $\Jav$ is the proved selection bound.\\
$\Jav$ & $\sum_{r=0}^R\E_{\Gamma_\mu}\lref_r(v_r)$, identified with \texttt{Experiment.availableTotal}. Its difference from $\Jal$ is bounded by $G$.\\
$\Jal$ & Aligned represented comparator score, expectation of \texttt{aligned}. Its difference from $\Jcl$ is bounded by $bS$.\\
$\Jcl$ & Score after removing alignment, expectation of \texttt{clean}. Its difference from $\Jpa$ is bounded by $\hresp(S+1)+D_C$.\\
$\Jpa$ & Reference execution-path score after response comparison, expectation of \texttt{refPath}. Its difference from $\Jrm$ is bounded by $E_{\rm nf}$.\\
$\Jrm$ & Reference comparator meta-policy score, expectation of \texttt{refM}. Its difference from $\Jnm$ is bounded by $E_C$.\\
$\Jnm$ & Expected native comparator loss $J_T(\mu)$, expectation of \texttt{realM}. No outgoing arrow is needed.\\
$E_A,E_C,E_{\rm nf}$ & Learner deployment, comparator deployment, and representation budgets, corresponding to \texttt{Budget.E\_cpA}, \texttt{E\_cpC}, \texttt{E\_nf}.\\
$F,G,D_A,D_C$ & Foreground-displacement, generation, learner-calibration, and comparator-calibration budgets. The Lean \texttt{Budget} fields have these same names.\\
$b,\hresp,N_b,S,V$ & Nominal run cap, response-bound scale, extra-boundary envelope, path-switch envelope, and scaled capacity. Lean fields are \texttt{b}, \texttt{H}, \texttt{N}, \texttt{S}, \texttt{V}. In particular, \texttt{Budget.N} is $N_b$, not the slot cap $N$.\\
$\mathcal P_b,\Delta,U_T$ & Prefix allowance, sum of seven scalar charges, and final upper bound, matching \texttt{Budget.prefix}, \texttt{delta}, and \texttt{bound}.\\
$g_T^{\rm gen}(\mu),G$ & Exact signed endpoint difference $\Jav-\Jal$ and a uniform upper budget. Their relationship is $g_T^{\rm gen}(\mu)\le G$, not equality by definition.\\
\end{longtable}

The seven directed certificates in~\eqref{eq:certificates} are not seven independently optimized regret objectives. Together with the proved selection arrow they form the eight adjacent differences in~\eqref{eq:telescope}. The scalar budgets may be signed when a tighter directed comparison permits it. The local generation theorem constructs a nonnegative envelope. No sign assumption is used to telescope the endpoints. Nonnegative multipliers are required when substituting upper budgets, such as $\eta/2$ for the capacity inequality.

\subsection{Local-model and benchmark notation}
\emph{Local generation reliability} is the collective term for the success and task-quality conditions in Section~\ref{sec:localgeneration}. Success includes validation, retention, and causal usability. The parameters $\rho_e$ and $\delta_e$ quantify these properties, while episode accounting and hard feasibility govern their composition. This term introduces no additional numeric parameter.
\begin{longtable}{@{}p{0.27\linewidth}p{0.68\linewidth}@{}}
\caption{Notation specific to stability, generation, and the paid binary benchmark.}\label{tab:notation-local}\\
\toprule Symbol & Meaning and scope\\\midrule
\endfirsthead
\toprule Symbol & Meaning and scope, continued\\\midrule
\endhead
\bottomrule\endfoot
$S_t,s,K_t,L_t,p^{\rm init}$ & Complete finite state, its realization, native and reference state-transition kernels, and common initial distribution in Section~\ref{sec:local}. The capital switch bound $S$ is not a state.\\
$c_t,V^L_{t,n},h_t,\bar H$ & Common immediate state score, $n$-step continuation value from state-time $t$, local continuation-range envelope, and its uniform bound. $\bar H$ is distinct from the response scale $\hresp$.\\
$\TV,\KL,\varepsilon_t,\bar\varepsilon$ & Total variation, relative entropy, local full-state TV envelope, and uniform activity-scaled TV factor. Relative entropy uses natural logarithms and the common support in Appendix~\ref{app:benchmark}.\\
$\chi_t,r_t,c_{\min},B_A,B_C$ & Nonnegative update activity, actual-path resource charge, strictly positive cost per unit activity, and learner/comparator resource envelopes. Here $r_t(s)$ is a resource function, not a slot number.\\
$\zeta_t,\mathcal D_{t,n},\kappa$ & Local product $h_t\varepsilon_t$, its accumulated $K$-Bellman charge, and $\bar H\bar\varepsilon/c_{\min}$. The proof integrates along the native $K$ law.\\
$e,k,k_e,\mathcal H_{e,k}^{\rm gen}$ & Generation episode, trial index, number of trials, and episode-specific completed-trial space. Trials and task rounds need not coincide.\\
$\mathbb P_e,\varpi_k^{\rm gen},g_k,b_k,\xi_k$ & Episode law, prefix mass, generation kernel, success indicator in $\{0,1\}$, and all-failed indicator. Episode index $e$ is suppressed inside the one-episode proof.\\
$\rho_e,\delta_e,\tau_e,n_e$ & Conditional retained-card success lower bound, post-success quality tolerance, number of delay tasks, and number of later tasks in that episode.\\
$g_t^{\rm pre},g_t^{\rm post},\xi_e$ & Conditional task-gap functions and failure indicator after $k_e$ trials. The first two are not the trial kernel $g_k$. Their episode sums identify the actual generation endpoint.\\
$A_T,\delta_T,f_T,U_T^{\rm rest}$ & Uniform delay, quality, and failure envelopes, and the reference bound with $G=E_A=E_C=0$. $A_T$ is a numeric delay allowance, not the learner $A$.\\
$\rho_e(\delta),c_e(\delta)$ & Quality-dependent conditional coverage and prepaid per-attempt resource upper bound in Corollary~\ref{cor:joint-budget}.\\
$B_T^{\rm other},C_A(h)$ & Reserved other costs and total learner cost at legal prefix $h$.\\
$C^{\rm other}(h),u_e(h)$ & Other charged costs and episode attempt count at legal prefix $h$.\\
$\overline U_T,\lambda_T$ & Full transfer bound with only $G$ removed, and a uniform lower bound on exposure $k_e\rho_e(\delta_e)$.\\
$F_{e,j},\mathcal F_{e,j}$ & No-success event after $j$ trials and completed-trial information in episode $e$. The event $F_{e,j}$ is distinct from the foreground-loss budget $F$.\\
$\Lambda,\rho_\star,c_\star$ & Log-failure target, coverage lower bound, and trial-cost envelope for the explicit quota in~\eqref{eq:joint-quota}.\\
$a_{\rm g},c_0,d_{\rm g},d_{\rm v},\nu$ & Polynomial coverage coefficient, cost coefficient, nonnegative integer exponents, and their positive sum in~\eqref{eq:joint-polynomial}.\\
$E_T,b_T^{\rm gen},z_T$ & Episode-slot count, attempt-resource allocation, and dimensionless ratio defining the positive quality root. None is the run cap $b$ or a deployment-error certificate.\\
$m,n,B,j,K_j,j(t)$ & Public phase count, phase length, total integer registration budget, phase index, phase library size, and phase containing task $t$. The benchmark uses $T=mn$.\\
$\mathcal D,\ell_t^{\rm base},I_t,q$ & Legal loss-table domain, fixed binary base loss, emitted identity, and primitive observation probability. The score masks are independent Bernoulli variables with parameter $q$.\\
$\mathcal V,\mathbf v,I_t^{\mathbf v}$ & Deterministic eligible phase paths, one such path, and its emitted identity. Unlike the native comparator, this benchmark comparator is a phase path on an oblivious table.\\
$\mathfrak A,\mathcal R^*_{\rm bin}$ & All causal randomized phase controllers and their minimax task regret in~\eqref{eq:minimax}.\\
$\mathbb P_0,\mathbb P_j,\epsilon,N_j$ & Null and hidden-target observation laws, Bernoulli mean gap, and the number of draws of identity $j$ in the finite lower-bound proof.\\
$p_t^{\rm low},x_j,y_j,E_q$ & Conditional learner kernel in that proof, null and target expected draw counts, and the independent-erasure map on distributions.\\
$\mathcal T_r,u_r,X_i,V_{\rm phase}$ & Task-index set of a held run, its length, its corrected loss under identity $i$, and phase-level second-moment envelope. Phase calculations suppress their phase index.\\
\end{longtable}

\subsection{What the correspondence check does and does not establish}
The paper notation is deliberately not forced to reproduce every internal Lean variable. In \texttt{DependentTransfer.Experiment}, fields \texttt{P}, \texttt{z}, \texttt{v}, and \texttt{V} denote $\Theta$, $\widehat\ell$, $\mathbf v$, and $V_{\rm est}$. Its measure parameter, locally named \texttt{\textmu} in mathematical presentations of the code, is the coupling law rather than a comparator meta-policy. The field \texttt{R} yields $R+1$ scored slots. In \texttt{FiniteObservation.Round}, \texttt{kernel}, \texttt{accepted}, \texttt{score}, \texttt{loss}, and \texttt{q} correspond to $k$, $M$, $y$, $\lref$, and $q$. The target loss is an input field and the accepted-score identity is a premise. An operational scorer $\lambda$ supplies a target instance, not a new theorem that every log is calibrated.

The native record convention follows \texttt{CausalTrajectory.History}, whose prefix includes coordinate zero. Task score $t$ is evaluated after record $t+1$, as in \texttt{NativeSemantics.Execution.score}. The resource class in \texttt{ResourceClasses} has $T+1$ task-boundary prefix coordinates. It does not silently quantify over an unrelated number of internal operations. Local generation indicators are ordinary $0/1$ values when multiplied, corresponding to Boolean tests with explicit coercion in the code.

The companion includes proved correspondence results without adding a numbered paper theorem. The family endpoint explicitly permits horizon-dependent finite domains and uses the combined-resource limit in Theorem~\ref{thm:local}. The source-and-log audit checks the included declarations and their dependencies. The interpretation of the English model and its operational identifications is manually reviewed, not mechanically proved by renaming notation. Integrability and exact score-law identifications remain part of the stated contracts.

The operational control kernel $\mu_A$ is not the literal type of \texttt{NativeSemantics.Execution.controller}. That Lean field is the induced task-record kernel after internal operations have been packaged together with their results and runtime effects. The interface assumes this packaging is well-defined. Renaming the operational kernel does not establish a new compiler-correctness theorem.

\subsection{Terminology conventions}
The \emph{reference algorithm} denotes Algorithm~\ref{alg:reference}, previously called a shadow algorithm. A \emph{reference experiment} additionally specifies its probability law, scores, and comparator path. A \emph{proxy score} is the observed scalar $y_r$, not a second algorithm. A \emph{run} is a scored execution interval, a \emph{slot} is its reference decision index, and a \emph{generation episode} groups attempts and associated task gaps. These units need not have equal lengths.

The complete-execution objective $R_T$ is called counterfactual regret to emphasize own-history policy-regret semantics. The signed generation opportunity gap $g_T^{\rm gen}(\mu)$ is one intermediate difference, and $G$ is its upper bound. Neither is a newly introduced independent regret objective. A certificate is an explicit bound used by the proof, while Lean verification checks a conditional implication. An assumed certificate is not automatically an empirically validated fact. The paid binary benchmark is a restricted theoretical model, not a dataset or experiment.

Algorithm~\ref{alg:reference} exposes initialization, active and padded slots, and the final weight update of~\eqref{eq:recursion}. The proof uses $N=R+1$ scored slots and the outgoing update at $r=R$ to form $W_{R+1}$. That update creates no extra task or comparator switch. Supplied values of $q_r$ and a lawful run protocol are interface inputs, not quantities automatically inferred from native logs. The pseudocode restates the analyzed recursion and does not claim a verified executable implementation of the native controller or generation scheduler.

\subsection{Reading horizon-dependent conditions correctly}
In a family of problems all uniform envelopes may depend on $T$. Uniformity means independent of the comparator at that fixed horizon, not constant as the horizon grows. To read Corollary~\ref{cor:native}, write these quantities as $b_T,h_{\rm resp,T},N_{b,T},S_T,C_T,V_T,\eta_T,\Delta_T$. For $C_T,V_T>0$, choosing $\eta_T=\sqrt{2C_T/(b_TV_T)}$ makes the selection term $\sqrt{2C_Tb_TV_T}$. Zero-capacity cases should use the original inequality, without dividing by zero.

For example, conditional on the required certificates, the scaling choices $h_{\rm resp,T}=O(1)$, $b_T=\lceil T^{1/3}\rceil$, $S_T=O(1)$, $N_{b,T}=o(T)$, $V_T=O(T)$, $C_T=O(\log(T+1))$, and $\Delta_T=o(T)$ make each normalized term vanish. The selection term is $O(T^{2/3}\sqrt{\log(T+1)})$. The nominal prefix term is $O(T^{2/3})$ and extra boundaries contribute $o(T)$. This illustrates the existing corollary, not an assertion that every growing agent supplies these envelopes. In particular, if births force $N_{b,T}$ to be linear, or the native/reference mismatch makes $\Delta_T$ linear, the example does not apply.

\begin{table}[ht]
\centering\small
\caption{Roles of the results in the conditional native-transfer analysis.}
\label{tab:roles}
\begin{tabularx}{\linewidth}{@{}p{0.23\linewidth}X@{}}
\toprule Component & Role and scope of the contribution\\\midrule
Growing selection & Classical exponential weights and sharing, checked on history-dependent finite types.\\
Native transfer & Directed own-law composition, with application-dependent interface certificates.\\
Local stability & Range-sensitive TV telescoping, specialized to actual-path adaptive resource charging.\\
Local generation & Conditional coverage and per-task quality bound the generation opportunity gap.\\
Paid benchmark & Explicit common-model information lower and primitive upper, near matching only under~\eqref{eq:sparse}.\\
Lean artifact & Reproducible verification of the conditional mathematical chain and its explicit interfaces.\\\bottomrule
\end{tabularx}
\end{table}

\subsection{Symbols local to the concrete reflection instance}
Appendix~\ref{app:concrete} uses $k$ for the number of generic repair attempts, $N$ for subsequent tasks and reference slots, and $T=k+N$ for the full task horizon. This scalar $k$ is not an observation kernel $k_r(i,o)$. The programs $P_0,P_1,P_2$ are immutable execution identities, not probability laws. The trial bits $\zeta_j$ form history $H_k$, and $\xi(h)$ is the all-failed indicator, consistent with the episode failure indicators used in the main generation result. The private bit $D$ concerns the single deployment check. The procedures $\mu_0,\mu_1$ form the own-history class $\mathcal M_0$. The finite masses $w_k,s,Q_{\eta,h}$ define the coupling $\Gamma$ in~\eqref{eq:instance-joint}. The family sets $k=n+1$, $N=k^2$, and $\eta=1/k$. No new global regret definition is introduced.

\section{Clock alignment and validated mathematical interfaces}
\label{app:v19-connections}
This appendix makes the interface boundary explicit. An interface is a premise about a specified model, not an unproved conclusion about its regret. The additional proofs below are checked in \texttt{CompletedClock}, \texttt{CompletedNativeV19}, and \texttt{ObservationBridgeV19}. They do not introduce an external controller.

\subsection{Completed tasks and state-score positions}
The native score $Y_t=\sigma_t(H_{t+1})$ is produced after task $t$. Take the finite sufficient state $S_{t+1}$ to contain the completed history, and write $g_t(S_{t+1})=Y_t$. In this appendix $g_t$ is the completed-task scorer, not a generation gap. Define the padded state scorer by
\begin{equation}
 \widetilde c_0(S_0)=0,\qquad
 \widetilde c_{t+1}(S_{t+1})=g_t(S_{t+1}).
 \label{eq:v19-padding}
\end{equation}
There are $T$ tasks and $T+1$ scored state positions. A terminal dummy transition after $S_T$ has zero resource cost and no effect on a scored task. The two processes share a pre-task initial state law, not the law of their first completed score.

For a kernel family $K$, let $W^K_{t,n}(s)$ be the expected score of the next $n$ completed tasks. Its recursion and its relation to the state-score value are
\begin{align}
 W^K_{t,0}(s)&=0,\qquad
 W^K_{t,n+1}(s)=\sum_jK_t(s,j)\bigl[g_t(j)+W^K_{t+1,n}(j)\bigr],
 \label{eq:v19-postvalue}\\
 V^{K,\widetilde c}_{t+1,n+1}(s)&=g_t(s)+W^K_{t+1,n}(s).
 \label{eq:v19-shift}
\end{align}
For $n=0$, the second identity says that one state position contributes its current score. For the induction step, expand one state transition and apply the same identity at each destination. Nonnegative normalized kernel weights preserve the equality. Expanding the finite trajectory law also proves that $W$ equals the expectation of the sum of completed-task scores. Consequently
\begin{equation}
 J_T^{\rm task}(K)
 =\E_{\mathbb P_K}\sum_{t<T}g_t(S_{t+1})
 =\E_{\mathbb P_K}\sum_{s=0}^{T}\widetilde c_s(S_s)
 =J_{T+1}^{\rm state}(K,\widetilde c).
 \label{eq:v19-clock}
\end{equation}
The last dummy transition integrates to one. Thus neither the first nor the last task score is removed. The Lean declarations are \texttt{post\_trace\_value}, \texttt{shifted\_value}, and \texttt{completed\_clock}.

Apply the state-score stability theorem at length $T+1$ using $\widetilde c$. Choose the terminal update identically on both sides, with zero local charge and resource. The actual task budget, generation task count, reference bound $U_T$, and regret normalization remain indexed by $T$. The resulting completed-task inequality is
\begin{equation}
 |J_T^{\rm task}(K)-J_T^{\rm task}(L)|
 \le \bar H\bar\varepsilon B_T/c_{\min}.
 \label{eq:v19-postgap}
\end{equation}
This is \texttt{CompletedClock.completed\_local\_gap}. With no influence improvement the safe padded local-only bound is $(T+1)^2\bar\varepsilon_T$. Its normalized sufficient condition is still $T\bar\varepsilon_T\to0$ as $T\to\infty$.

\subsection{A validated instance and the exact composition}
The publication-level \texttt{CompletedNativeV19.Instance} requires unit-range completed-task scores, positive task and run counts, nonnegative resources and activities, zero terminal resources, and integrability of every intermediate expected-score endpoint. The local state contracts use $T+1$ positions. The reference budget uses $T$ actual tasks. No equality between the two clocks is silently imposed.

The remaining reference, episode, and own-law identities are still explicit. For each fixed comparator, the generation theorem gives its budget $G=A_T+T(\delta_T+f_T)$. Instantiate the reference certificates with this budget and zero inner deployment arrows. Convert the two reference identities using~\eqref{eq:v19-clock}, invoke the local transfer at the padded state horizon, and convert the two native endpoints back to completed tasks. The reference task bound is unchanged. Thus
\begin{equation}
 R_T\le U_T^{\rm rest}+A_T+T(\delta_T+f_T)
       +\kappa(B_A+B_C),\qquad \kappa=\bar H\bar\varepsilon/c_{\min}.
 \label{eq:v19-composed}
\end{equation}
The checked endpoint is \texttt{CompletedNativeV19.Instance.bound}. Its variable-domain family theorem invokes this bound before taking limits, using the actual task denominator $T$. It replaces the unpadded family endpoint as the publication correspondence for Theorem~\ref{thm:local}. Horizon-dependent model families need not be consistent prefixes of one anytime implementation. No such stronger implementation claim is made.

\subsection{Connecting local observations to the same experiment law}
Let $h$ range over finite pre-draw histories, with probabilities $a(h)$. At each $h$, a finite observation round has the paper's accepted-score contract, draw probability $p_h$, and outcome kernel $k_h$. A law witness maps a complete sample $\omega$ to its history, draw, and outcome. For every finite test function $f$ it requires
\begin{equation}
 \E_{\Gamma_\mu} f(h,J,O)
 =\sum_h a(h)\sum_i p_h(i)\sum_o k_h(i,o)f(h,i,o).
 \label{eq:v19-disintegration}
\end{equation}
This specifies the joint law, not merely an identity marginal. Histories can have different finite identity types. Inactive slots can use a singleton absorbing event. For use with the selection recurrence, $p_h$ must be its actual pre-draw probability and the estimator must be its actual update. The comparator coordinate is a function of $h$, not of the current outcome.

Apply the already proved local first-moment identity separately at each $h$, multiply by $a(h)$, and sum. Apply the local second-moment inequality in the same way, using $a(h)\ge0$. This gives
\begin{align}
 \E_{\Gamma_\mu}\bigl[\langle p,\widehat\ell\rangle-\widehat\ell(v)\bigr]
 &=\E_{\Gamma_\mu}[\ell^{\rm ref}(J)-\ell^{\rm ref}(v)],
 \label{eq:v19-lawmoment}\\
 \E_{\Gamma_\mu}\langle p,\widehat\ell^2\rangle
 &\le\sum_h a(h)d_h^2\sum_iq_h(i)^{-1}.
 \label{eq:v19-lawsecond}
\end{align}
\texttt{LawWitness.moment}, \texttt{LawWitness.second}, and \texttt{identified\_moment} verify these implications and the pointwise pullback identification. The law witness and recurrence identification must be instantiated for a concrete model. There is no claimed automatic constructor from arbitrary native logs to a growing reference execution. The general reference experiment continues to expose these law and moment interfaces. This is the precise boundary of the conditional verification claim.

\subsection{Accounting is determined by endpoints}
For three endpoints, $a-c=(a-b)+(b-c)$, checked by \texttt{adjacent\_not\_causal}. For example, native, reference, and ideal expected losses $T,T/2,0$ require both adjacent charges $T/2$. Both can be consequences of generation failures without being duplicate charges. Only charging the same endpoint difference twice is prohibited. Conversely, an outer full-state charge already includes propagation of its own perturbation, so those same consequences must not be added again to that same arrow. A physical resource may constrain feasibility and affect a task-score difference, but feasibility alone does not add another loss penalty.

\subsection{Joint feasibility and the scope of sublinearity}
From $\operatorname{TV}(K_t,L_t)\le\bar\varepsilon\chi_t$ and $c_{\min}\chi_t\le r_t$, finite summation gives
\begin{equation}
 \sum_{t<T}\operatorname{TV}(K_t(S_t,\cdot),L_t(S_t,\cdot))
 \le \bar\varepsilon B_T/c_{\min}.
 \label{eq:v19-tvbudget}
\end{equation}
This is \texttt{accumulated\_tv}. Also, zero resource implies zero activity and zero local TV, as proved by \texttt{no\_resource\_no\_error}. Thus charging only rare reflection calls cannot cover persistent differences in ordinary selection. Constant-scale discrepancies on a positive fraction of rounds are incompatible with a sublinear resource envelope at fixed $\bar\varepsilon,c_{\min}$.

Coverage attempts, delays and hard budgets must be feasible in the same instance. Corollary~\ref{cor:instance} supplies a concrete native-transfer witness using direct deployment certificates. The stronger full-state range contracts of the local-resource variant still require their own justification. Corollary~\ref{cor:joint-budget} derives feasible quotas under quality-dependent coverage and cost contracts. Neither result turns separately desirable exponents into a universal model guarantee. Trial resource use must follow from the same legal design.

For the separate sparse benchmark, divide $q^3T\le mB\log B$ by the positive $qT$ to obtain
\begin{equation}
 q^2\le\frac{mB\log B}{qT}.
 \label{eq:v19-sparse}
\end{equation}
The finite inequality is \texttt{sparse\_normalized}. It shows why the sparse-family sufficient rate condition requires vanishing observation probabilities. It is not an impossibility result for fixed positive observation probability. The public exogenous registration schedule also does not establish necessity of every endogenous native certificate.

\section{Related work and research connections}
\label{app:related}
\paragraph{Availability and generation.}
Sleeping bandits address availability~\citep{saha2020,nguyen2024}, ballooning models address arriving arms~\citep{ghalme2020}, and growing-expert tracking aggregates expanding pools~\citep{mourtada2017}. Generative action sets use a context-conditioned oracle, permanent reuse, quadratic mismatch and a doubly-optimistic creation-selection controller~\citep{xu2026gas}. Our Theorem~\ref{thm:local} connects causally usable policies to complete native executions through delay, quality and conditional failure. It permits dependent trials. Strongly adaptive regret offers an interval-based perspective~\citep{daniely2015strongly}, while our episodes partition the available-to-aligned endpoint difference.

\paragraph{History, feedback, and switching.}
Classical bandit theory connects specified controllers and observation laws to regret~\citep{lattimore2020,bubeck2012}. Policy regret~\citep{arora2012}, memory-dependent losses~\citep{anava2015,cesabianchi2013}, and Markov games~\citep{nguyentang2025} ground history-responsive comparisons. Our transfer retains the native operation-selection kernel. Theorem~\ref{thm:stability} bounds full-state discrepancies along actual resource paths. Theorem~\ref{thm:local} combines this charge with local generation bounds through ordered score endpoints.

Switching-cost bandits~\citep{dekel2014switch,rouyer2021}, feedback-sensitive bounds~\citep{cheng2023feedback}, and feedback graphs~\citep{alon2015} characterize the role of observation and action-change protocols. Our reference analysis specifies accepted-score calibration and proves the associated observation moments. The supporting model of Theorem~\ref{thm:benchmark} uses independent primitive erasures and paid public growth, giving near-matching bounds in the stated sparse regime with common learner, comparator, and cost conventions.

\paragraph{Mechanized learning guarantees.}
\citet{merten2018} verified multiplicative weights and composed distributed-system guarantees in Coq. Our Lean artifact connects native execution, growing selection, observation, generation, deployment and completed-task clocks. It verifies concrete native instances and a joint quality-budget design. Declaration mappings identify the contracts and endpoints of all thirteen numbered results.

Rule-card identities select action kernels, connecting exponential weighting to execution semantics. Full-state stability and conditional failure bounds connect local reflection to total loss. Table~\ref{tab:roles} summarizes these component roles.

\paragraph{Connection to native-agent evaluation.}
Empirical studies of native LLM online play~\citep{park2025} motivate analysis of the controller used by the agent. Theorems~\ref{thm:stability} and~\ref{thm:local} connect local generation reliability and full-state deployment stability to the complete native objective. Their joint control determines when the conditional bound is informative.

Off-policy evaluation~\citep{dudik2011} supports analysis of logged decisions. Our reference process additionally requires an identified history law, observation support and accepted-score calibration. Appendix~\ref{app:application} connects these operational requirements to the transfer interfaces.

\section{Repeated growth with imperfect validation}
\label{app:imperfect-growth}
This second native instance makes validation and subsequent selection genuinely fallible. New context-specific policies are generated in several phases. A false acceptance can overwrite a previously correct native choice and cause errors throughout the remaining phase. The controller follows a recency rule, never runs the reference Bandit, and has no deterministic repair prompt. The construction leaves Theorems~\ref{thm:native} and~\ref{thm:local} unchanged. It provides a sharper direct calculation and separately verifies compatibility with the existing transfer interfaces.

\subsection{A frozen generator and an imperfect executable validator}
Fix integers $E,k,m,N\ge1$ and $T=E(k+N)$. There are $E$ public task types, presented in consecutive phases. In phase $e$, each deployment task has input $x=1$ and requests output one. Candidate programs are immutable and type-tagged. A correct type-$e$ program returns $x$ on type $e$ and zero on other types. An incorrect candidate returns zero. The initial fallback also returns zero. Correct programs for different types induce different action kernels. Previously registered cards remain in memory, but the legal execution interface calls only current-type cards or the fallback.

The same frozen generation kernel, queried with the current type, independently produces a correct or incorrect candidate with probability $1/2$. No quality label is returned to the controller. A generated handle names immutable content and its version. Each attempt has a fresh handle, even when its behavior duplicates another candidate. Semantic growth across types is distinguished below from the raw number of registrations.

The validator tests a candidate on $m$ independent fair inputs in $\{0,1\}$ and accepts only when all outputs match their inputs. A correct current-type program always passes. An incorrect program passes only when every probe is zero. Thus its conditional false-acceptance probability is
\begin{equation}
 \beta_m=2^{-m}>0.
 \label{eq:imperfect-beta}
\end{equation}
This is an imperfect finite-test validator, not a complete program verifier. The sampling protocol is fixed. The native procedure receives candidate handles, test inputs and outputs, and the acceptance decision. It never receives the latent correct/incorrect label. All procedures use the same generation and validation operations. The environment in this instance is deterministic apart from the specified generation and validation randomness.

For analysis, a completed attempt has one of three outcomes, denoted $0,1,2$, with probabilities
\begin{equation}
 w_m(0)=\tfrac12,\qquad w_m(1)=\tfrac{\beta_m}{2},\qquad
 w_m(2)=\tfrac{1-\beta_m}{2}.
 \label{eq:imperfect-outcomes}
\end{equation}
Outcome zero is a correct accepted card, one an incorrect accepted card, and two a rejection. These probabilities follow by pushing the fair generation bit and the actual $m$ probe bits through the validator. Outcomes zero and one have the same acceptance signal. The controller can therefore use the acceptance bit without observing this three-way analysis label.

\subsection{Native execution, retention, and hard resources}
In each phase, the native learner executes the fallback for the first $k$ tasks and performs one generation attempt after each task. Every accepted candidate is immediately registered and retained. Each acceptance updates a current-type pointer to that newest handle. Rejections preserve the pointer. The following $N$ tasks execute the pointer, or the fallback if none was accepted. The pointer is not repaired during this suffix, even when a task reveals an error. This is an explicit native batching rule, not a claim about every reflective system.

Registration does not force immediate deployment. The batch choice pays for its $k$ preparation tasks. The interface still permits other native schedules, including immediate use after a completed registration and a new predictable draw. The construction does not impose a block-boundary restriction on the general model.

Generation costs one unit, each probe costs one unit, and each accepted registration costs one unit. Before initiating an attempt, its worst-case cost $m+2$ is reserved. Generation and probe costs of failed attempts remain charged. Any unused registration reserve is released only after the attempt completes. Every prefix, including an unfinished attempt, then satisfies
\begin{equation}
 C_A(h)\le\sum_{e<E}(m+2)u_e(h)
 \le Ek(m+2)=B_T,\qquad 0\le u_e(h)\le k.
 \label{eq:imperfect-hard}
\end{equation}
Here $u_e(h)$ is the number of initiated attempts in phase $e$. No registration is possible before initiation. The common prefix cap is $B_T$, available from the start. Stronger time-varying resource schedules would require their own feasibility checks. Resources constrain legal execution and are not added a second time to the task loss.

The comparator class $\mathcal M$ is any nonempty class of complete legal hard-feasible procedures with the same task scorer. For the separation from no reflection, require that it contains the following proactive procedure. At the start of each announced phase, it completes the same $k$ random generation and validation attempts, uses the same latest-accepted rule, and then executes all $k+N$ tasks. It pays the same worst-case budget and has neither a stronger generator nor a perfect validator. Its observations and memory are generated on its own execution, not copied from the learner.

\begin{proposition}[Repeated native growth with imperfect validation]
\label{prop:imperfect-growth}
The specified native procedure is hard-feasible under~\eqref{eq:imperfect-hard}. Let $q_{m,k}$ be the probability that its phase pointer is incorrect after $k$ attempts, including the fallback. Then
\begin{equation}
 q_{m,0}=1,\qquad
 q_{m,j+1}=\tfrac{\beta_m}{2}+\tfrac{1-\beta_m}{2}q_{m,j},\qquad
 2^{-k}\le q_{m,k}\le2^{-k}+\beta_m.
 \label{eq:imperfect-recursion}
\end{equation}
The complete own-law objective satisfies
\begin{equation}
 J_T(A)=E[k+Nq_{m,k}],\qquad
 R_T(A,\mathcal M)\le E[k+N(2^{-k}+2^{-m})].
 \label{eq:imperfect-bound}
\end{equation}
The expected number of task types for which a genuinely correct retained policy has been generated is $E(1-2^{-k})$. If $k,m\ge2$ and the proactive comparator belongs to $\mathcal M$, the no-reflection procedure has regret at least $T/2$. With $E=k=m\to\infty$ and $N=k^3$, both positive normalized native regret and normalized hard resources tend to zero.
\end{proposition}

\subsection{Proof from the actual execution law}
Let $H_j$ be the completed outcome history in one phase. Its law is recursively generated from~\eqref{eq:imperfect-outcomes}. Let $Z_j(H_j)\in\{0,1\}$ indicate whether the selected pointer is incorrect. Initially $Z_0=1$. An outcome-zero acceptance sets it to zero, an outcome-one acceptance sets it to one, and a rejection retains the old value. Importantly, the control decision in both acceptance cases is identical, namely point to the newest handle. Only analysis of the program output distinguishes the cases.

To connect this recursion to executable program content, give attempt $j$ the local handle $j$ and pair it with its phase tag. The registry stores accepted handles without deleting older entries. Induction on $j$ proves that every nonempty pointer names an accepted handle with index less than $j$. A second induction proves that the correctness of the program at that pointer is exactly $1-Z_j$. Rejection preserves both the old pointer and the immutable content at its index. Acceptance points to the new content, irrespective of its latent quality. Thus $Z_j$ is the actual program loss at deployment input one, not an independently assigned surrogate score. The same registry induction proves that a correct retained program exists exactly when at least one outcome-zero attempt has occurred.

Conditioning on the next outcome $O_{j+1}$ after the completed prefix gives
\begin{equation}
 \mathbb E_{O_{j+1}\mid H_j}[Z_{j+1}]
 =\tfrac{\beta_m}{2}+\tfrac{1-\beta_m}{2}Z_j.
 \label{eq:imperfect-step}
\end{equation}
Integrating this equality yields the recursion in~\eqref{eq:imperfect-recursion}. This computation derives $q_{m,j}$ from the native pointer update rather than assuming a small selection error. For the upper bound, suppose $q_{m,j}\le2^{-j}+\beta_m$. Since $0\le\beta_m\le1$,
\begin{align}
 q_{m,j+1}
 &\le\tfrac{\beta_m}{2}+\tfrac{1-\beta_m}{2}(2^{-j}+\beta_m)
 \label{eq:imperfect-induction-a}\\
 &=2^{-(j+1)}+\beta_m
   -\beta_m2^{-(j+1)}-\tfrac{\beta_m^2}{2}
 \le2^{-(j+1)}+\beta_m.
 \label{eq:imperfect-induction-b}
\end{align}
The initial inequality holds because $q_{m,0}=1\le1+\beta_m$. If no correct candidate was generated, the pointer must be incorrect. This event has probability $2^{-k}$, proving the lower bound in~\eqref{eq:imperfect-recursion}.

For every realized phase history, its actual sum of task scores is $k+NZ_k(H_k)$. Summing over the recursively generated law gives $k+Nq_{m,k}$. The full execution retains old libraries but each current-type kernel ignores their contents, so composing the phase laws gives the first equality in~\eqref{eq:imperfect-bound}. Equivalently, linearity of expectation applied to the actual phase rewards gives the same identity. Since every complete comparator has nonnegative common task loss, its own-law expected loss is nonnegative. Subtracting the nonnegative infimum proves the regret upper bound. This argument applies to the original legal comparator class, rather than granting the comparator an unavailable card or a zero-loss oracle.

Every generated correct card is accepted and retained. Consequently a type has a correct retained policy exactly when at least one of its $k$ generation bits is correct. The expected number of such types is $E(1-2^{-k})$. This counts semantically different type-specific policies, not duplicate handles within one type. The expected total number of accepted registrations, including duplicates and incorrect cards, is separately $Ek(1+\beta_m)/2$.

The proactive comparator uses the same random pointer distribution but performs preparation before its tasks. Its own-law objective is
\begin{equation}
 J_T(\mu_{\rm pro})=E(k+N)q_{m,k}=Tq_{m,k}.
 \label{eq:imperfect-proactive}
\end{equation}
For $k,m\ge2$, the upper bound on $q_{m,k}$ is at most $1/2$. The no-reflection procedure has loss $T$ and therefore regret at least $T-J_T(\mu_{\rm pro})\ge T/2$. This comparator is another noisy reflective procedure with its own generation history, not an omniscient policy. All resource prefixes of both procedures obey the same cap.

\subsection{Why the validation error is not an instantaneous charge}
Let $M_k$ indicate that no correct card was generated. Define a clean analysis execution that, after the same preparation history, uses a retained correct card if one exists. Then $M_k\le Z_k$ for every history, and
\begin{equation}
 \mathbb E_{H_k}[M_k]=2^{-k},\qquad
 0\le\mathbb E_{H_k}[Z_k-M_k]\le\beta_m.
 \label{eq:imperfect-gap}
\end{equation}
The clean execution is not the deployed controller and is not assumed to run a perfect validator. It is an intermediate comparison object. Its advantage over the native pointer is proved from the actual three-outcome law and charged explicitly.

For example, a correct acceptance followed by an incorrect acceptance has probability $\beta_m/4$. The correct card remains registered, but the native pointer now calls the incorrect one. In fact, for two attempts the excess error is exactly
\begin{equation}
 q_{m,2}-2^{-2}=\beta_m/4>0.
 \label{eq:imperfect-positive}
\end{equation}
Thus an available correct policy does not make the native selection bridge free. A wrong pointer remains active throughout $N$ suffix tasks, so the deployment edge is at most $N\beta_m$ per phase. Counting only $\beta_m$ would miss this continuation effect. The growing family below compensates by improving validation as the suffix grows. It does not assume that the global continuation-value range is one.

\subsection{Connection to the unchanged transfer theorem}
Condition on a completed phase history before constructing any reference draw. Use the same two-arm reference recursion as Appendix~\ref{app:concrete}, with losses one and $M_k$. The second arm represents an actually retained correct card when one exists, and an incorrect fallback otherwise. The clean native suffix has loss $NM_k$, which is no larger than the reference expected selection loss. The reference initialization, feedback law, code $\log2$, and second-moment budget $2N$ are the previously proved ones. The finite controller form of Theorem~\ref{thm:native} therefore has the following phase certificates
\begin{equation}
 F=k,\qquad G=N2^{-k},\qquad E_A\le N\beta_m,\qquad
 E_C=D_A=D_C=E_{\rm nf}=0,
 \label{eq:imperfect-ledgers}
\end{equation}
with response and switch charges zero inside this phase comparison. The all-correct ideal path has value zero, a valid lower comparison endpoint for any nonnegative own-law comparator value. It is not claimed to be a feasible zero-loss comparator. The direct calculation avoids paying the additional reference selection bound, whereas transfer gives the valid, generally looser inequality
\begin{equation}
 R_T(A,\mathcal M)\le
 E\left[k+N(2^{-k}+2^{-m})+\frac{\log2}{\eta}+\eta N\right],
 \qquad \eta>0.
 \label{eq:imperfect-transfer}
\end{equation}
Each phase uses an independently constructed reference analysis after its preparation has completed. These restarts do not reset the native library. We do not silently invoke a single two-arm fixed-library process across all growing phases.

More explicitly, if $w_{m,k}(h)$ is the phase history law and $Q_{\eta,h}(g)$ the actual reference recursion law, define $\Gamma(h,g)=w_{m,k}(h)Q_{\eta,h}(g)$. Summing over $g$ returns $w_{m,k}(h)$, and the native task score does not depend on $g$. Thus the reference extension preserves the complete native phase marginal. Its phasewise composition preserves the full native execution used above. Comparator objectives retain their respective own laws. This is a causal auxiliary construction, not offline reconstruction from unsupported action logs.

The three costs in~\eqref{eq:imperfect-ledgers} have different endpoints. Preparation losses occur before the suffix. Generation loss is the suffix failure event $M_k=1$. Additional deployment loss is the event $Z_k=1,M_k=0$. These last two events are disjoint. Failed validation and candidate registration consume resources in~\eqref{eq:imperfect-hard}, not an additional loss penalty. This prevents charging the same false acceptance as both generation failure and deployment discrepancy.

\subsection{An explicit family with growing libraries and sublinear resources}
Take $k\ge2$, $m=E=k$ and $N=k^3$. Then
\begin{equation}
 T=k^4+k^2,\qquad B_T=k^2(k+2),\qquad
 \frac{[R_T(A,\mathcal M)]_+}{T}\le k^{-2}+2^{1-k},\qquad
 \frac{B_T}{T}\le k^{-1}+2k^{-2}.
 \label{eq:imperfect-family}
\end{equation}
Both right sides tend to zero. The first follows by dividing~\eqref{eq:imperfect-bound} by $T\ge k^4$. The resource calculation uses the same denominator, not a different budget design. The expected number of distinct repaired task types is $k(1-2^{-k})$, which grows without bound. For the transfer route, choosing $\eta=1/k$ adds at most $(\log2)k^{-2}+k^{-1}$ to the normalized direct bound, also tending to zero.

The guarantee is a horizon-indexed family, not an anytime claim. The restrictions are explicit, including independent template sampling, one-sided validation noise, public type-local phases and no cross-type effect on future generation. The instance supplies repeated semantic growth and persistent imperfect-validation errors within these restrictions. It does not establish joint tightness of the general upper bound or automatic guarantees for arbitrary language-model validators.

\subsection{Formal correspondence}
Module \texttt{ImperfectGrowth} connects the native registry and program loss to the probability model through \texttt{actual\_execution}. Its \texttt{actual\_regret} and \texttt{actual\_no\_regret} prove the bound and limit, allowing horizon-dependent comparator classes. The \texttt{full\_transfer} proof invokes the unchanged transfer through the phase certificates. Appendix~\ref{app:map} and the declaration mapping cover the intermediate validation, retention and resource proofs. This verifies the specified mathematical program, not deployed language-model code. Reference learning is not necessary for the stronger direct calculation.

\end{document}